\documentclass[twoside]{article}

\usepackage[accepted]{aistats2020}
\usepackage[round]{natbib}

\input{definition}
\newcommand{\head}{\paragraph}
\begin{document}

%

%

\twocolumn[

\aistatstitle{PAC-Bayesian Transportation Bound}

%

\aistatsauthor{Kohei Miyaguchi}

\aistatsaddress{ IBM Research }

]

\begin{abstract}
Empirically, the PAC-Bayesian analysis is known to produce tight risk bounds for practical machine learning algorithms.
However, in its \naive form, it can only deal with stochastic predictors
while such predictors are rarely used and deterministic predictors often performs well in practice.
To fill this gap,
we develop a new generalization error bound,
the \emph{PAC-Bayesian transportation bound},
unifying the PAC-Bayesian analysis and the chaining method in view of the optimal transportation.
It is the first PAC-Bayesian bound that relates the risks of any two predictors
according to their distance, and capable of evaluating the cost of de-randomization
of stochastic predictors faced with continuous loss functions.
As an example, we give an upper bound on the de-randomization cost of spectrally normalized neural networks~(NNs) to
evaluate how much randomness contributes to the generalization of NNs.
\end{abstract}

\section{INTRODUCTION}

The goal of statistical learning is to acquire the predictor $f\in\Fcal$ that (approximately) minimizes a risk function $r(f)$
in a inductive way.
In doing so, one is only allowed to access some proxy function $\rhat_S(f)$ based on noisy data $S$.
Therefore, the goal of \emph{statistical learning theory} is to describe the behavior of the deviation of the proxy from the true risk,
\begin{align}
    \Delta_S(f)\eqdef r(f)-\rhat_S(f).
    \label{eq:deviation_zakkuri}
\end{align}
In particular, we are interested in computable high-probability upper bounds on $\Delta_S(f)$, say $\bar{\Delta}_S(f)$,
to bound the true risk with a computable function, $r(f)\le \rhat_S(f)+\bar{\Delta}_S(f)$.

The PAC-Bayesian analysis~\citep{mcallester1999some,catoni2007pac}
is one of the frameworks of such statistical learning theory
based on the strong duality of the Kullback--Leibler~(KL) divergence~\citep{donsker1975asymptotic}.
Below, we highlight two major advantages of the PAC-Bayesian approach.
\begin{itemize}
    \item[\bf (A1)] It is tight and transparent.
        The gap of the resulting bound is optimal in the sense of the strong duality.
        Moreover, one can easily interpret the meaning of each term in the bound and where it comes from.
    \item[{\bf (A2)}]
        It is usable in practical situations.
        It has been also confirmed repeatedly in the literature that
        it produces non-vacuous bounds on the generalization of complex predictive models
	such as deep and large-scale neural networks~\citep{dziugaite2017computing,zhou2018nonvacuous}.
\end{itemize}
In particular, the second point explains well the recent growing attention attracted on the PAC-Bayesian theory
in the machine learning community~\citep{neyshabur2018a,dziugaite2018entropy,mou2018generalization,nagarajan2018deterministic}.

The issue we address here is that it cannot handle deterministic predictors.
More precisely,
it provides upper bounds on the \emph{expectation} of the deviation function with respect to distributions $\Qb$ over $\Fcal$,
namely $\tilde{\Delta}_S(\Qb)\ge \EE_{f\sim \Qb}[\Delta_S(f)]$,
and the upper bound diverges for almost all the deterministic settings $\Qb=\deltab_f$.
This is problematic in the following two viewpoints.
\begin{itemize}
    \item[{\bf (I1)}]
        The stochastic predictors are rarely used in practice.
        \footnote{
            Although some algorithms, including stochastic gradient descent, are stochastic in their nature,
            they are often not used in the way the PAC-Bayesian framework suggests.
        }
        The use of deterministic predictors is fairly common
        and their performances are not even close to what is predicted by the \naive PAC-Bayesian theory.
    \item[{\bf (I2)}]
        More importantly, the contribution of the stochasticity is unexplained.
        In previous studies, it has been pointed out repeatedly from both empirical and theoretical perspectives that
        introducing stochasticity into prediction (sometimes drastically) improves
 	the predictive performance~\citep{welling2011bayesian,srivastava2014dropout,neelakantan2015adding,russo2015much}.
        However,
        it is unclear
        whether it is also the stochasticity that makes possible the tightness of the PAC-Bayesian bounds or not,
        since deterministic predictors cannot be accurately described with the \naive PAC-Bayesian theory.
\end{itemize}

To address these issues {\bf (I1)} and {\bf(I2)},
we present a new theoretical analysis that unifies the PAC-Bayesian analysis and the \emph{chaining} method.
The chaining~\citep{dudley1967sizes,talagrand2001majorizing}
 is the technique that gives the tightest known upper bound (up to a constant factor) on the supremum of the deviation function, $\sup_{f\in \Fcal} \Delta_S(f)$,
and is understood as a process of discretization refinement
starting from finitely discretized models $\ddot\Fcal\subset \Fcal$ to reach the limit of continuous models $\Fcal$.
We extends this idea to the process of noise shrinking
starting from stochastic predictors $\Qb$ to reach the limit of deterministic predictors $\deltab_f$.
As a result, we obtain the risk bound, namely \emph{PAC-Bayesian transportation bound},
which is interpreted as the KL-weighted cost of transportation over predictor space $\Fcal$.

In particular, our contribution is summarized as follows.
\begin{itemize}
    \item[{\bf (C1)}]
    The proposed bound is the first general bound that allows us to relate
    the risks of any two stochastic or deterministic predictors in general.
    More specifically, it allows us, for the first time,
    to assess the effect of stochasticity in prediction
    by comparing stochastic predictors with deterministic ones.
    \item[{\bf (C2)}]
    To demonstrate the effectiveness of the bound, we give an upper bound of the noise reduction cost of neural networks~(NNs) and the corresponding de-randomized generalization error bound.
    The resulting risk bound is as tight as the conventional PAC-Bayesian risk bound for randomized NNs~\citep{neyshabur2018a}
    up to a logarithmic factor,
    and hence indicates that the stochasticity is not essential
    in this specific setting.
\end{itemize}

%

The rest of the paper is organized as follows.
First, the problem setting is detailed in Section~\ref{sec:problem_setting}.
Then,
in Section~\ref{sec:main_result},
the main result is presented with a proof sketch.
The interpretation and comparison to previous studies are also included here.
In Section~\ref{sec:examples}, we utilize the proposed bound for analysing the risk of neural networks.
Finally, we give several concluding remarks in Section~\ref{sec:conclusion}.

\section{PROBLEM SETTING}
\label{sec:problem_setting}

In this section, we first overview the PAC-Bayesian framework
to specify our focus, and then
introduce the mathematical notation
to describe the precise problem we are concerned with.

\subsection{The PAC-Bayesian Framework in a Nutshell}

The goal of the PAC-Bayesian analysis is to derive high-probability upper bounds on the deviation function $\Delta_S(f)$ given by~\eqref{eq:deviation_zakkuri}.
The difficulty to achieve this goal is that predictors $f$ are learned from data $S$ and thus
there occurs a nontrivial statistical dependency among those two variables.
To avoid this problem, the PAC-Bayesian theory suggest that one follows two key steps, namely, the \emph{linearization} and \emph{decoupling}.

In the first step, linearization,
the predictors are generalized to be stochastic,
i.e.,  upon prediction, a random predictor is drawn from some probability measure $\Qb\in \Pi(\Fcal)$,
called a \emph{posterior}, which can depend on the data in a nontrivial way.
The performance of such stochastic predictors is measured by the expectation $\Qb r\eqdef \int r(f) \Qb(\d f)$,
and hence the deviation is also studied in the form of expectation, $\Qb \Delta_S$.
Note that the ordinary deterministic predictors are special instances of stochastic ones
as they are recovered by taking Dirac's delta measures $\Qb=\deltab_f$.
In this way, any kinds of interactions between $f$ and $S$ are now
formulated as the bilinear pairing of a predictor $\Qb$ and a data-dependent function $\Delta_S$.

Now, in the second step,
the bilinear pairing is decoupled with the Fenchel--Young type inequalities,
namely $\Qb \Delta_S \le \zeta(\Qb)+\zeta^*(\Delta_S)$,
which allow us to deal with predictors and data separately.
Here, $(\zeta, \zeta^*)$ denotes a pair of Fenchel conjugate functions.
Specifically, the standard PAC-Bayesian analysis
exploits the strong duality between the KL divergence and the log-integral-exp function,
\begin{align}
    \Qb \Delta_S \le \beta^{-1} \kld(\Qb, \Ub) + \beta^{-1} \ln \Ub \sbr{e^{\beta\Delta_S}},
    \label{eq:decoupling}
\end{align}
$\beta > 0$,
where $\Ub\in\Pi(\Fcal)$ is any data-independent distribution called a \emph{prior}~(see Appendix~\ref{sec:app_basic_pac_bayes} for the proof).
Finally, the data dependent term, $\beta^{-1}\ln \Ub[e^{\beta \Delta_S}]$, is bounded with the concentration inequalities,
such as Hoeffding's inequality and Bernstein's inequality,
and we obtain high-probability upper bounds on $\Qb\Delta_S$ as desired.

\head{Our Focus}
Unfortunately, the bound~\eqref{eq:decoupling} is meaningless
when $\Qb$ is not absolutely continuous with respect to $\Ub$, since the KL term diverges.
In particular, if one takes non-atomic priors $\Ub$, e.g., Gaussian measures,
then, it diverges with any Dirac's delta posteriors $\Qb=\deltab_f$.
More generally, when the model space $\Fcal$ is continuous,
then almost every deterministic predictors are prohibited to use under the \naive PAC-Bayesian bound~\eqref{eq:decoupling}.
This is the problem we focus on in this paper.

\subsection{Mathematical Formulation}

\head{Conventions} For any measurable spaces $\Xcal$, we denote by $\Pi(\Xcal)$ the space of probability distributions over $\Xcal$.
For any two-ary function $X:(f,z)\mapsto X(f, z)$, let $X(f)$ and $X(z)$ denote the partially applied functions such that
$X(f):z\mapsto X(f, z)$ and $X(z):f\mapsto X(f, z)$.
Moreover, if the function $X$ is measurable,
we denote its expectation with respect to $\Qb\in \Pi(\Fcal)$ and $\Pb\in\Pi(\Zcal)$ by $\Qb\Pb{X}=\Pb\Qb{X}\eqdef \iint X(f, z) \Qb(\d f)\Pb(\d z)$,
where $\Fcal$ and $\Zcal$ are arbitrary measurable spaces.
We also reserve another notation for expectations; The expectation with respect to any random variables $f\sim \Qb$ maybe denoted by $\EE_{f\sim \Qb}$,
or just $\EE_f$ if any confusion is unlikely.

\head{Basic Notation and Assumptions}
Let $z_i\in\Zcal~(i=1,\ldots,n)$ be \iid random variables corresponding to single observations subject to unknown distribution $\Pb\in\Pi(\Zcal)$.
Let $S=(z_1, \ldots, z_n)\in\Zcal^n$ be the collection of such variables, $S\sim \Pb^n\in\Pi(\Zcal^n)$, 
from which we want to learn a good predictor.
Let $\Fcal$ be a measurable space of predictors,
such as neural networks and SVMs with their parameters unspecified,
and let $f\in \Fcal$ denote predictors with specific parameters.
We assume that $\Fcal$ is (a subset of) a separable Hilbert space with inner product $\inner{\cdot}{\cdot}$ and norm $\abs{\cdot}$,
e.g., $d$-dimensional parameter spaces or infinite-dimensional function spaces.

Let $\ell(f, z)\in \RR$ be the loss of the prediction made by predictors $f\in \Fcal$ upon observations $z\in\Zcal$,
accompanied with the Fr\'echet derivative $\nabla\ell(f,z)$ with respect to $f\in\Fcal$ for all $z\in\Zcal$.
Also, we define the \emph{risk} of the predictors $f\in\Fcal$ by $r(f)\eqdef\Pb \ell(f)$.

To facilitate the analysis of transportation in later, we also assme that $\Fcal$ is endowed with an inverse metric $\Sigma:\Fcal\to \Scal_+(\Fcal)$,\footnote{
    Equipped with $\Sigma$, $\Fcal$ can be thought of as a Riemannian manifold. However, we do not assume the smoothness nor the invertibleness of $\Sigma(f)$
    as standard Riemannian manifolds do.
}
where $\Scal_+(\Fcal)$ is the set of symmetric nonnegative linear operators on $\Fcal$.
It defines a local norm $\abs{\cdot}_{\Tcal_f}$ at each point $f\in\Fcal$, $\abs{v}_{\Tcal_f}=\sqrt{\inner{v}{\Sigma^{-1}(f)v}}$
for all $v\in \Sigma(f) \Fcal$
and otherwise $\abs{v}_{\Tcal_f}=\infty$.
The metric is used to bound the variation of $\Delta$.

\begin{assumption}[Lipschitz condition]
    \label{asm:lipschitz_base}
    The deviation function $\Delta(f,z)$ is $L_\Delta$-Lipschitz continuous with respect to $f\in\Fcal$, i.e.,
    \begin{align*}
        &\lim_{\rho \to 0} \sup_{\abs{g-f}_{\Tcal_f}\le \rho}\frac{\abs{\Delta(g,z)-\Delta(f,z)}}{\rho}
        \le L_\Delta
    \end{align*}
    for all $f\in\Fcal$ and $z\in \Zcal$.
\end{assumption}

Note that the standard Lipschitz condition is recovered if $\Sigma(f)$ is identity for all $f\in \Fcal$.\footnote{
    For example, deviation $\Delta$ of squared loss function $\ell(f,z)=\frac12|z-f|^2$, $\Fcal=\RR$, satisfies Assumption~\ref{asm:lipschitz_base}
    with $\Sigma(f)\equiv 1$ and $L_\Delta=1$
    if $z\in [-1,+1]$.
}
However, by appropriately choosing $\Sigma$, we can handle a broader class of deviation functions beyond standard Lipschitz ones.

\head{Problem Statement}
Our objective here is to find the predictor $f$ with small risk $r(f)$.
However, since $r$ is inaccessible as $\Pb$ is,
we leverage the empirical risk measure $\rhat_S$ to approximate the true objective $r$.
Let $\Pb_S=\frac1n\sum_{i=1}^n\deltab_{z_i}$ be the empirical distribution with respect to the sample $S$.
Then, the empirical risk of $f\in\Fcal$ is given by
$
    \rhat_S(f)=\Pb_S \ell(f)=\frac 1n\sum_{i=1}^n \ell(f, z_i),
$
which is a random function whose expectation coincides with the true risk, $\EE_{S\sim \Pb^n} [\rhat_S] \equiv r$.
As it fluctuates around the mean, we are motivated to study the tail probability of the deviation
$\Delta_S=r-\rhat_S$.
Define the \emph{deviation function} (of single observation) by
\begin{align}
    \Delta(f, z)&\eqdef r(f)-\ell(f, z).
    \label{eq:deviation}
\end{align}
Since $\Delta_S=\Pb_S \Delta$, we want to find a tight high-probability upper bound on the sample-averaged deviation of posterior distributions $\Qb\in\Pi(\Fcal)$ in the form of
\begin{align}
    \Qb\Pb_S\Delta
    &\le U(\Qb, S) + Z(S),
    \label{eq:goal}
\end{align}
where $U$ is a computable function and
$Z(S)$ is a negligible random variable independent of $\Qb$ satisfying that $\mathop{\rm Pr}\{Z(S)>0\}\le \delta$ with some confidence level $0<\delta<1$.

To conclude this section,
we introduce the Kullback--Leibler~(KL) divergence,
which is used to measure the complexity of predictors in the PAC-Bayesian analysis.

\begin{definition}[The KL divergence]
    Let $\Qb,\Ub\in \Pi(\Fcal)$ where $\Fcal$ is a measurable space.
    The KL divergence between $\Qb$ and $\Ub$ is given by
    \begin{align*}
        \kld(\Qb,\Ub)\eqdef \Qb\sbr{\ln \frac{\d \Qb}{\d\Ub}},
    \end{align*}
    where $\Qb$ is absolutely continuous with respect to $\Ub$.
    Otherwise, $\kld(\Qb,\Ub)=+\infty$.
    Moreover,
    \begin{align*}
        H_\delta(\Qb,\Ub)\eqdef \kld(\Qb,\Ub)+\ln \frac1\delta
    \end{align*}
    denotes the PAC-Bayesian complexity of posteriors $\Qb$ with respect to priors $\Ub$
    with confidence level $\delta\in(0,1)$.
\end{definition}

\section{MAIN RESULT}
\label{sec:main_result}

In Section~\ref{sec:preliminary}, we introduce the analytical tools necessary for stating our main result.
Then we present the main result with a few remarks in Section~\ref{sec:main_of_main}.
Finally, in Section~\ref{sec:main_discussion}, we compare it with relevant existing results and give a proof sketch.
The rigorous proof is included in Appendix~\ref{sec:app_main_proof}.

\subsection{Posterior Flow and Velocity Measures}
\label{sec:preliminary}

Our goal is to find the upper bounds in the form of~\eqref{eq:goal} applicable to both stochastic and deterministic predictors $\Qb\in \Pi(\Fcal)$.  To this end, we combine the PAC-Bayesian framework with the chaining method.
The chaining is understood as a process of relating one predictor $f_i$ to its neighbors $f_{i+1}$ and moving towards
some destination $f_\infty$ through the \emph{chain} $f_1\to f_2\to f_3\to \cdots\to f_\infty$.
We extend this idea and take the infinitesimal limit $\sup_{i}d(f_i, f_{i+1})\to 0$ where $d$ is an appropriate distance over $\Fcal$.
As a result, we get the continuous \emph{transportation} of predictors instead of the chain.

To facilitate the analysis of the transportation of posterior distributions,
we introduce ordinary differential equations~(ODE) over $\Fcal$ that transport posteriors,
which we call \emph{posterior flow}.
Let $t\in T\eqdef [0, \infty)$ be a time index indicating the progress of transportation.
Let $\Qb_0\in \Pi(\Fcal)$ be the initial posterior distribution.
Let $\xi=\cbrinline{\xi_t:\Fcal\otimes \Omega_0 \to \Fcal}_{t\in T}$
be a random time-indexed vector field on $\Fcal$,
where $(\Omega_0, P_0)$ is a probability space representing the source of randomness in transportation itself.

\begin{definition}[Posterior Flow]
    \label{def:ODE}
    We say $(\Qb_0, \xi)$ is a posterior flow
    if the solution of the following ODE $\cbr{f_t}_{t\in T}$ exists almost surely,
    \begin{align}
        \d f_t &= \xi_t(f_t, \omega_0) \d t,\quad f_0 \sim \Qb_0, \ \omega_0\sim P_0,
        \label{eq:ODE}
    \end{align}
    and the corresponding snapshot distributions
    $\cbrinline{\Qb_t^\xi\in\Pi(\Fcal)}_{t\in T}$ are well-defined,
    i.e.,
    $f_t\sim \Qb_t^\xi$ for all $t\in T$.
    Moreover, $\mu=\cbrinline{\mu_t}_{t\in T}$ is the mean posterior flow of $\xi$ if
    \begin{align*}
        \mu_t(f) = \EE_{f_0\sim \Qb_0,\omega_0\sim P_0}\sbr{\xi_t(f, \omega_0) \;\middle|\; f_t=f}.
    \end{align*}
\end{definition}

Note that all the posteriors $\Qb_t^\xi$, $t\in T$, are completely identified
once we specify the initial condition $\Qb_0$ and the flow $\xi$.
When it is clear from the context,
we may omit $\Qb_0$ and refer to $\xi$ as a posterior flow.
Let
\begin{align*}
    D_t(\xi;S)\eqdef \Qb_t^\xi \Pb_S \Delta
\end{align*}
be the deviation function of the posterior generated with $\xi$ at time $t$.
In our analysis, the deviation $D_t(\xi;S)$ is characterized
with respect to two velocity measures of posterior flows.

The first one is given by the 2-Wasserstein distance~\citep{villani2008optimal}
between $\Qb_t^\xi$ and $\Qb_{t+h}^\xi$ at the limit of $h\to0$.
\begin{definition}[Wasserstein velocity]
    The Wasserstein velocity of $\xi$ at time $t\in T$ is defined as
    \begin{align*}
        W_t(\xi) \eqdef  \sqrt{\EE_{f\sim \Qb_t^\xi} \abs{\mu_t(f)}_{\Tcal_f}^2}.
    \end{align*}
\end{definition}

The Wasserstein velocity is determined only by the metric structure of the predictor space $\Fcal$,
which indirectly reflects the continuity of the deviation function $\Delta$ through the Lipschitz condition~(Assumption~\ref{asm:lipschitz_base}).
On the other hand, the second velocity measure reflects the structure of $\Delta$ in a more direct manner.

\begin{definition}[Deviation-based velocity]
    The deviation-based velocity of $\xi$ with respect to $S\in\Zcal^n$ at time $t\in T$ is defined as
    \begin{align*}
        V_t(\xi;S) 
        &\eqdef \sqrt{\Qb_t^\xi \frac{\Pb+\Pb_S}{2}\inner{\mu_t}{\nabla\Delta}^2}
        = \sqrt{\Qb_t^\xi \inner{\mu_t}{\Lambda_S \mu_t}},
    \end{align*}
    where $\Lambda_S(f)\eqdef \frac{\Pb+\Pb_S}{2} \nabla\Delta(f)\otimes \nabla \Delta(f)$.
\end{definition}

Note that we have $0\le V_t(\xi;S)\le L_\Delta W_t(\xi)$.
This reveals that the roles of these two velocities are indeed complementary to each other;
While $V_t(\xi;S)$ is tighter and offers finer characterization of the posterior flow $\xi$ in a data- and distribution-dependent way,
it is guaranteed to be bounded by $W_t(\xi)$ in the worst case just in the same way as loss functions are bounded by constant in the standard PAC-Bayesian analysis.

\subsection{PAC-Bayesian Transportation Bound}
\label{sec:main_of_main}

To evaluate the deviation function $D_t(\xi;S)$, we utilize the fundamental theorem of calculus,
\begin{align}
    D_t (\xi;S)
    &= D_0(\xi;S) + \int_{0}^{t} \frac{\d}{\d u} D_u(\xi;S) \d u.
    \label{eq:FTC}
\end{align}
This motivate us to seek for an upper bound on the infinitesimal increment $\frac{\d}{\d t} D_t(\xi;S)$.

\begin{theorem}[Transportation bound]
    \label{thm:transportation_bound}
    Fix any prior distribution $\Ub\in \Pi(\Fcal)$.
    Then, 
    the increment of the deviation is bounded by 
    \begin{align}
        \frac{\d}{\d t}D_t(\xi;S)
        &\le
            2V_t(\xi;S)
            \sqrt{
                \frac{H_\delta(\Qb_t^\xi,\Ub)+c(n)}{n}
            }
            +
            \frac{2L_\Delta W_t(\xi)}{\sqrt{n}}
        \nonumber\\
        &=: \iota_t(\xi;S,\Ub,\delta)
        \label{eq:transportation}
    \end{align}
    with probability $1-\delta$ on the draw of $S\sim \Pb^n$
    for simultaneously all posterior flow $\xi$ and all $t\in T$.
    Here,
    $c(n)\eqdef \ln (e+\ln^2(2n^2))=\Ocal(\ln \ln n)$.
\end{theorem}
The proof is given in Appendix~\ref{sec:app_main_main}.

\head{Remark~(Interpretation)}
The upper bound~\eqref{eq:transportation} consists of two terms, each of which is interpreted in a different way.

Ignoring the sub-polynomial factor $c(n)$,
the first term is the deviation-based velocity $V_t(\xi;S)$ times the model complexity per sample $\sqrt{{H_\delta(\Qb,\Ub)/n}}$.
This is analogous to the key quantity of the chaining bound known as Dudley's entropy integral,
where $H_\delta(\Qb,\Ub)$ is corresponding to the metric entropy.
Thus, it can be interpreted as the \emph{chaining cost}.
Note that, under some regularity conditions,
the entropy term $H_\delta(\Qb_t^\xi,\Ub)$ is roughly evaluated to be of the same order with the dimensionality of model $\dim \Fcal=d$.
Since the velocity factor $V_t(\xi;S)$
is bounded by $L_\Delta W_t(\xi)$ in the worst case,
the integration over $T$ recovers
the traditional dimensionality-dependent bounds $\Ocal(\sqrt{d/n})$
as long as the length of the transportation $\xi$ is bounded with respect to 2-Wasserstein distance.
On the contrary, our bound can be significantly tighter than traditional ones if $V_t(\xi;S)$
is much smaller than $L_\Delta W_t(\xi)$
and/or the entropy term $H_\delta(\Qb_t^\xi,\Ub)$ is much smaller than the dimensionality $d$.
In particular, the effect of $V_t(\xi;S)$ is unique to our bound;
The results shown in Section~\ref{sec:examples} is owing to the fact that
$V_t(\xi;S) / W_t(\xi)=\Ocal(\sqrt{K/m})$
with $m$ the number of neurons per layer and $K$ the depth of neural networks.

On the other hand, the second term is proportional to $W_t(\xi)$.
Therefore, after integrating it over $T$,
it is proportional to the 2-Wasserstein length of posterior transportation,
hence it can be thought of as a pure \emph{transportation cost}.
It is negligible in comparison to the first term since
it is independent of the complexity of the model $\Fcal$ such as
the dimensionality.
Moreover, we note that the bound~\eqref{eq:transportation}
is not sensitive to the choice of metric $\Sigma$
since it does not appear in the chaining-cost term.

\head{Remark (Distance induced by optimal flow)}
The inequality \eqref{eq:transportation} induces a metric on the space of posterior distributions
$\Pi(\Fcal)$.
If we integrate it over $t$ and optimize it with respect to $\xi$ with the source $\Qb_0^\xi$ and the destination $\Qb_1^\xi$ fixed,
a distance function over $\Pi(\Fcal)$ is given as
\begin{align}
    d_{S,\Ub,\delta}(\Qb, \Qb')
    &\eqdef \inf_{\substack{\xi:\Qb_0^\xi=\Qb,\Qb_1^\xi=\Qb'}}
    \int_0^1 \iota_t(\xi;S,\Ub,\delta) \; \d t
    \label{eq:optimal_transportation_distance}
\end{align}
for all $\Qb,\Qb'\in \Pi(\Fcal)$,
which we call \emph{the optimal transportation~(OT) distance} of predictors.
Putting this back to \eqref{eq:FTC} with Theorem~\ref{thm:transportation_bound}
and decomposing the deviation function, we obtain a relative risk bound in terms of the closeness of two predictors.
\begin{corollary}[Transportation-based risk bound]
    \label{cor:transportation_based_risk_bound}
    Fix a prior $\Ub\in\Pi(\Fcal)$.
    Then,
    \small
    \begin{align*}
        \underbrace{\Qb r}_{\text{\tiny\rm\bf Risk}}
        &\le \underbrace{\Qb \rhat_S}_{\text{\tiny\rm\bf Emp.\;Risk}}
        + \underbrace{d_{S,\Ub,\delta}(\Qb, \Qb_0)}_{\text{\tiny\rm\bf OT distance}}
        + \underbrace{\Qb_0 \Pb_S \Delta}_{\text{\tiny\rm\bf Reference deviation}}
    \end{align*}
    \normalsize
    with probability $1-\delta$
    for simultaneously all posteriors $\Qb_0,\Qb\in \Pi(\Fcal)$.
\end{corollary}

This corollary has two distinct implications.
The first implication is that (i) the OT distance bounds the risk of predictors by itself.
If the reference predictor $\Qb_0$ is fixed,
then the reference deviation is of order $\Qb_0\Pb_S\Delta=\Ocal(n^{-1/2})$,
independent of the complexity of $\Fcal$,
and just negligible.
Hence the deviation of $\Qb$ is shown to be governed by the OT distance
from an arbitrary fixed predictor.
For example, if $\Delta(f, z),~z\sim \Pb$ is $\sigma$-subgaussian, then we have
\begin{align}
    \Qb \Pb_S \Delta \le d_{S, \Ub,\delta/2}(\Qb, \Ub)+\sigma\sqrt{\frac{\ln 2/\delta}{n}},
    \label{eq:transportation_based_deviation_bound}
\end{align}
taking $\Qb_0=\Ub$ in Corollary~\ref{cor:transportation_based_risk_bound}.

The second implication is that (ii) the OT distance can be used to describe the cost of de-randomization.
Consider $\Qb_0$ as any stochastic predictors
and let $\Qb$ be any ``less stochastic'' predictors.
Then, Corollary~\ref{cor:transportation_based_risk_bound} implies
a trade-off relationship of the empirical fitness and the OT distance.
Specifically, the fitness $\Qb\rhat_S$ is likely to get worse if the amount of noise contained in the prediction $\Qb$ is increased, whereas the OT distance $d_{S,\Ub,\delta}(\Qb_0,\Qb)$ can be decreased as $\Qb_0$ is stochastic.
As a result, the optimal amount of the noise can be determined by minimizing the RHS,
or one may simply take $\Qb$ to be deterministic, $\Qb=\deltab_f$, paying the cost of $d_{S,\Ub,\delta}(\Qb_0,\deltab_f)$.
This further implies that the transportation bound and the conventional PAC-Bayesian bound can be combined to
produce better risk bounds that cannot be achieved by themselves alone.

Note the OT distance is intractable in general due to the infimum.
We discuss the problem of the computational tractability in the next remark.
Moreover, in Section~\ref{sec:examples},
we give an example of upper bounds on it.

\head{Remark (Computational tractability)}
We note that \eqref{eq:transportation} and \eqref{eq:optimal_transportation_distance}
are computationally tractable with some numerical approximation methods.

As for the increment bound~\eqref{eq:transportation},
the only inaccessible quantity is $V_t(\xi;S)$, or more specifically, the data-dependent the metric $\Lambda_S$ in it.
Recall that we have a trivial computable upper bound, $L_\Delta W_t(\xi)$.
To get tighter bounds, one must exploit the problem-dependent structure of the loss function $\ell$ that characterizes $\Lambda_S$.

Once we get such upper bound, we may evaluate $\iota_t(\xi;S,\Ub,\delta)$ with Monte Carlo sampling approximation, choosing an appropriate posterior flow $\xi$.
To illustrate this, consider the transportation from an arbitrary initial distribution $\Qb_0\in\Pi(\Fcal)$ to the delta measures $\deltab_{f_0}$, $f_0\in\Fcal$.
One of the simplest such posterior flows is the linear contraction flow $\xi_t(f)=\mu_t(f)=f_0-f$.
In this case, the posterior $\Qb_t^\xi$ is nothing but the initial one $\Qb_0$ shrunk by a factor of $e^t$
towards $f_0$.
Therefore, as long as we can draw samples from $\Qb_0$, the velocities in the increment bound can be evaluated as $W_t^2(\xi)=e^{-2t}\;\EE_{f\sim \Qb_0} \abs{f_0-f}_{\Tcal_f}^2$ and so on.

On the other hand,  to compute the noise reduction cost~\eqref{eq:optimal_transportation_distance},
one has to evaluate the integral of $\int \iota_t(\xi;S,\Ub,\delta) \d t$ and take the infimum over $\xi$.
As for the infimum, we can just ignore it and compute an upper bound with a concrete instance of $\xi$.
As for the integral,
one may approximate the integral with finite sums.
More precisely, the time line is discretized with $0=t_0<t_1<\ldots<t_K=t$
and the integral is approximated with the summation $\sum_{k=1}^K (t_k-t_{k-1})\iota_{t_k}(\xi;S,\Ub,\delta)$.

\subsection{Comparison with Existing Bounds}
\label{sec:main_discussion}

In this subsection,
we discuss the difference between Theorem~\ref{thm:transportation_bound} and related existing results.

\head{\cite{audibert2007combining}}
An attempt to handle deterministic prediction within the PAC-Bayesian framework has been made earlier by \cite{audibert2007combining}.
In particular, they had already accomplished the goal of tightly bounding the risks of the deterministic predictors
with the PAC-Bayesian analysis assisted with the idea of chaining.
However, it cannot be utilized to relate conventional PAC-Bayesian bounds with deterministic predictors.

The essential difference is that their result is based on the chaining flow over the minimum covering tree of $\Fcal$,
whereas ours is based on the one over the entire predictor space $\Fcal$ with any direction as long as it can be expressed in the form of ODE.
This entails three apparent differences among two.

Firstly, because of the freedom in transportation flows,
we have to include the additional cost $\iota^{\rm T}_t(\xi)$,
which does not appear in the previous bound.
However, its impact is not serious
because it costs at most $\Ocal(n^{-1/2})$ without any dependency on the model complexity,
provided the Wasserstein length of posterior transportation
is bounded.\footnote{This is likely to be the case if the diameter of model $\Fcal$ is bounded.}
This is also confirmed in the example given in Section~\ref{sec:examples}.

Secondly, in the previous result,
the initial point of chaining should be a fixed deterministic predictor $f_0\in\Fcal$
because the flow has to be tree-shaped, i.e., there must be no more than one root point.
Therefore, it is not directly applicable for relating the deviations of two (stochastic or deterministic) predictors on the basis of their closeness, as we have done in Corollary~\ref{cor:transportation_based_risk_bound}.

Finally, the previous bound contains the KL divergence between \emph{discretized} posterior and prior distributions,
where the discretization is based on the minimum $\epsilon$-nets of the predictor space $\Fcal$.
Thus, it is difficult to evaluate their bound directly in practice.
On the other hand, our increment bound can be evaluated once a computable upper bound on $\Lambda_S$ is given.

\head{Chaining Method (Proof Sketch)}

We also compare our bound with the conventional chaining bound.
This gives a rough sketch of how we prove the main theorem~(Theorem~\ref{thm:transportation_bound}).

We start with a new insight on the essence of the chaining bound,
which forms the foundation of our bound.
The chaining is basically a sophisticated way of applying union bounds.
As the union bound can be thought of as a subset of the PAC-Bayesian bound~(e.g., take the prior as a counting measure and the posterior as a Dirac's delta),
it must also workaround the problem of the diverging phenomenon with \eqref{eq:decoupling}.

The key idea of chaining is to divide and conquer.
More precisely, instead of applying the Fenchel--Young inequality directly,
we first decompose the deviation function into a telescoping sum,
\begin{align}
    \Qb\Pb_S \Delta
    = \Qb_0 \Pb_S\Delta+\sum_{i=1}^\infty (\Qb_{i}-\Qb_{{i-1}})\Pb_S\Delta,
    \label{eq:telescope}
\end{align}
where the posterior sequence $\cbr{\Qb_i}$ is constructed to satisfy $\Qb_i\to\Qb$ as $i\to \infty$
in the sense of weak convergence.
This is the `dividing' part.

As for the `conquering' part,
we handle each of the summands separately.
Let $\Qb_{i,j}\in \Pi(\Fcal^2)$ be a joint distribution of a pair of predictors $f_i$ and $f_j$
whose marginals are corresponding to $\Qb_i$ and $\Qb_j$ respectively,
i.e., $f_i\sim \Qb_i$ and $f_j\sim \Qb_j$.
Also, let $X_S:(f, g)\mapsto \Pb_S (\Delta(g) - \Delta(f))$ be the increment function of $\Pb_S\Delta$.
Then, the summands can be seen as the bilinear pairing of $\Qb_{i-1,i}$ and $X_S$.
Applying the Fenchel--Young inequality with a series of conjugate pairs
$(\zeta_i, \zeta_i^*)~i=1,2,\ldots$,
we have
\small
\begin{align}
    \Qb\Pb_S \Delta -  \Qb_0\Pb_S\Delta
    &=
    \sum_{i=1}^\infty \Qb_{{i-1},i} X_S
    \nonumber
    \\
    &\le
    \sum_{i=1}^\infty \cbr{\zeta_i(\Qb_{{i-1},i}) + \zeta_i^*(X_S)}.
    \label{eq:conquer}
\end{align}
\normalsize
As a result, with an appropriate choice of joint distributions $\Qb_{i,j}$ and the conjugate series~(i.e., the way of applying union bounds),
the diverging behavior of the KL divergence is averaged out within the infinite summation.

On the other hand,
Theorem~\ref{thm:transportation_bound}
is an infinitesimal version of the conquering part, bounding
\begin{align*}
    \frac{\d}{\d t}D_t(\xi;S)=\lim_{u\to 0}\frac{(\Qb_{t+u}^\xi-\Qb_t^\xi)\Pb_S\Delta}{u},
\end{align*}
whereas the dividing part is owing to the fundamental theorem of calculus~\eqref{eq:FTC},
which is the continuous counterpart of \eqref{eq:telescope}.
The derivative $\frac\d{\d t}D_t(\xi;S)$ is then bounded with the Fenchel--Young inequality
in the same spirit of \eqref{eq:conquer},
where 
the joint distributions $\Qb_{i-1,i}$
are turned into
the posterior flow $\cbr{\mu_t}_{t\ge 0}$ utilizing the chain rule
\begin{align*}
    \frac{\d}{\d t}D_t(\xi;S)
    &=\Qb_{t}^\xi\inner{\mu_t}{\Pb_S\Delta}.
\end{align*}

\head{Standard PAC-Bayesian Bounds}
We also highlight two differences in our bound compared to the standard PAC-Bayesian bound.

Firstly, of course, it allows us to avoid the diverging KL phenomenon.
Note that the conventional PAC-Bayesian risk bound claims that
\begin{align}
    \Qb\Pb_S\Delta\le \tilde\Ocal\rbr{\sigma \sqrt{\frac{H_\delta(\Qb,\Ub)}{n}}},
    \label{eq:conventional_pac_bound}
\end{align}
where $\sigma$ is a variance-like scale factor of $\Delta$.
On the other hand,
our bound~\eqref{eq:transportation_based_deviation_bound} is roughly equivalent (ignoring the model-independent terms) to
\small
\begin{align}
    \Qb\Pb_S\Delta\le \tilde\Ocal\rbr{\int_0^1 \d t\; V_t(\xi;S) \sqrt{\frac{H_\delta(\Qb_t^\xi,\Ub)}{n}}},
    \label{eq:our_pac_bound_simplified}
\end{align}
\normalsize
where $\xi$ is taken to be the transportation
from $\Qb_0^\xi=\Ub$ to $\Qb_1^\xi=\Qb$.
Since the velocity $V_t(\xi;S)$ measures the rate of change in the deviation function in the $L^2$-sense,
it plays a similar role as that of $\sigma$.
Therefore, the essential difference is that the posterior can change over time in~\eqref{eq:our_pac_bound_simplified}.
As a result, the effect of the entropy $H_\delta$ is averaged and it remains finite even if $H_\delta(\Qb_t^\xi,\Ub)\to H_\delta(\Qb,\Ub)=\infty$ as $t\to 1$.
This is how our bound works around the diverging problem the conventional bound~\eqref{eq:conventional_pac_bound} suffers from.

Secondly, our bound can be significantly tighter than conventional ones even if $H_\delta(\Qb,\Ub)$ does not diverge.
This is explained with how these two bounds react to the change of the diameter $R$ of the model space $\Fcal$.
In the conventional bound,
$\sigma$ is not necessarily related to $R$
as it measures the scale of the absolute value of the deviation function $\abs{\Delta}$.
On the other hand,
the scaling factor of our bound, $V_t(\xi;S)$, linearly scales with the diameter
since the velocity of transportation reflects the distance over $\Fcal$ directly.
Therefore, if $R$ is sufficiently small, our bound can be much tighter than conventional bounds.



\section{DE-RANDOMIZING SPECTRALLY NORMALIZED NEURAL NETWORKS}
\label{sec:examples}
We demonstrate the effectiveness of the transportation bound
by recovering the risk bound of
spectrally normalized neural networks presented by \cite{neyshabur2018a} under weaker assumptions.

Let $a_k:\RR^{m}\to \RR^{m}~(k=1,\ldots,K)$ be a sequence of $1$-Lipschitz activation functions
satisfying the homogeneity condition $a_k(\alpha x)=\alpha a_k(x)$~(consider the ReLU activation for example).
Let $\Fcal$ be a set of $K$-depth neural networks $f:\RR^{m}\to \RR^{m}$
such that
$f_w(x)=W_K \circ a_K \circ \cdots \circ W_1\circ a_1(x)$,
where $w=\{W_k\in \RR^{m\times m}:1\le k\le K\}$ is a collection of matrices.
Note that the total dimensionality of the networks is $d=m^2K$.
To introduce the structure of Hilbert space into $\Fcal$,
we identify the network $f_w\in \Fcal$ with its parameter $w\in \RR^{d}$
(hence $\abs{f_w}=\abs{w}$ is the Frobenius norm of the Kronecker product $W_1\otimes \cdots \otimes W_K$).
Let $\Xcal\subset \RR^{m}$ and $\Ycal$ be the space of the inputs and the teacher signals, respectively,
and let $\Zcal=\Xcal\times \Ycal$ be the space of observations.
We denote by $R_\Xcal=\sup_{x\in\Xcal}\abs{x}$ the maximum scale of inputs.
Assume that the loss function is given by $\ell(f, (x,y))=\ell(y, f(x))$,
where $\ell(y, \cdot)$ is a $L_\ell$-Lipschitz continuous function defined on $\RR^{m}$, such as hinge loss and logistic loss.

We also introduce two characteristics of networks essential to the subsequent analysis.
Let 
$\Lbar(w)\eqdef (\prod_{k=1}^K \norm{W_k}_2)^{1/K}$
be the depth-normalized spectral radius
and let 
$R(w)\eqdef L_\ell R_\Xcal \Lbar^K(w)$
be the total spectral radius of $f_w$.
Here, $\norm{A}_2\eqdef \sup_{h\in \Fcal,h\neq 0}\abs{Ah}/\abs{h}$ denotes the operator norm of $A$.

Now, as an application of the transportation bound,
we present an upper bound on the de-randomization cost of Gaussian posteriors on $\Fcal$.
For simplicity,
we only consider the spherical Gaussian distributions with some scale correction factors,
\begin{align}
    \Nb(w, \rho)\eqdef \Ncal\rbr{w, \rho^2\frac{\Lbar^2(w)}{mK^2\gamma_m^2}I_d}, 
    \label{eq:stochastic_predictor_NN}
\end{align}
where
$\gamma_m=\sqrt{2\ln 2em}$,
$w\in \RR^d$ denotes the mean and $\rho> 0$ denotes the normalized scale.
Then, we connect the deviation of $\Nb(w, \rho)$ with
that of the deterministic counterpart $\deltab_w=\Nb(w, 0)$.

\begin{theorem}[De-randomization cost of NNs]
    \label{thm:noise_reduction_for_neural_network}
    Consider the stochastic predictor $\Nb(w, \rho)$ given in \eqref{eq:stochastic_predictor_NN}.
    Then, there exists a prior $\Ub\in\Pi(\RR^d)$ such that
    \small
    \begin{align*}
        &d_{S,\Ub,\delta}\rbr{\Nb(w,\rho), \deltab_w}
        \\
        &\le
        4 e^\rho R(w)
        \sqrt{\frac{mK^2}{n}}
        \sbr{
            \frac{\abs{w}}{\Lbar(w)} \;\Ical\rbr{\frac{\sqrt{m}\rho}{K\gamma_m}} +
            \frac{
                \rho
            \rbr{
                1+\sqrt{\frac{c_2}{m}}
            }
            }{K\gamma_m}
        }
        ,
    \end{align*}
    \normalsize
    for all
    $S\in\Zcal^n$ and $\delta\in (0,1)$.
    Here, we define $\Ical(a) \eqdef \int_0^a \d s\; \sqrt{\ln(1+s^{-2})}$
    and $c_2\eqdef  \ln \frac{e\sqrt{d}}{\delta}(e+\ln^22n^2)\rbrinline{\sqrt{\frac{m}{K}}\frac{\Lbar(w)\rho}{\gamma_m}+\sqrt{\frac{K}{m}}\frac{\gamma_m}{\Lbar(w)\rho}}$.
\end{theorem}
The proof is found in Appendix~\ref{sec:app_nn_noise_reduction}.
Theorem~\ref{thm:noise_reduction_for_neural_network} tells us how much additional cost we have to pay
if we are to remove the noise of scale $\rho$.
Note that the first term is dominant if $m$ is large and $\rho$ is moderately small.
Specifically,
taking $\rho = 1$,
we have
\small
\begin{align*}
    d_{S,\Ub,\delta}\rbr{\Nb(w, \rho^2), \deltab_w}
    &=\tilde\Ocal\rbr{R(w)\frac{\abs{w}}{\Lbar(w)}\sqrt{\frac{mK^2}{n}}}
\end{align*}
\normalsize
since $\Ical(a)=\Ocal(\ln a)$.


Combining the above result with existing PAC-Bayesian bounds,
we can prove deviation bounds of the deterministic NNs;
Inserting the result of Theorem~\ref{thm:noise_reduction_for_neural_network}
with $\rho=1$ to that of Corollary~\ref{cor:transportation_based_risk_bound},
we obtain the following risk bound.

\begin{corollary}[Risk of spectrally normalized NNs]
    \label{cor:total_deviation_of_neural_networks}
    Suppose that $0\le \ell(f_w, z)\le 1$ 
    for all $w\in\RR^d$ and $z\in\Xcal$.
    Then, for large $m$,
    \begin{align*}
        \Pb_S \Delta(f_w)
        &=
        \tilde\Ocal\rbr{
            R(w)\frac{\abs{w}}{\Lbar(w)}\sqrt{\frac{mK^2}{n}}
        }
    \end{align*}
    with high probability
    for simultaneously all $w\in \RR^d$.
\end{corollary}
The proof is also given in Appendix~\ref{sec:app_nn_risk_bound}.
Note that the confidence level parameter $\delta$ is still there,
but erased within the order notation $\tilde\Ocal$.

The bound is much smaller than the VC-dimension-based bound $\tilde\Theta(\sqrt{dK/n})$~\citep{harvey2017nearly}
when $m\gg K$ and
$\frac{\abs{w}}{\Lbar(w)}$ is moderate.
This is the case when the singular values of each weight matrix decay fast~(i.e., almost low rank).
However, in the worst case with the identical singular values, we have $\frac{\abs{w}}{\Lbar(w)}=\sqrt{mK}$
and the bound is not tighter anymore. 

Note that this is the same rate with the state-of-the-art PAC-Bayesian bound~\citep{neyshabur2018a}, which is derived under the margin assumption on $\ell(y, \cdot)$ to remove the stochastic noise of PAC-Bayesian predictors.
On the other hand, our requirement on the loss functions is the $L_\ell$-Lipschitz continuity,
which is strictly weaker than the margin assumption.
In other words,
the corollary implies that the stochasticity of the PAC-Bayesian predictor is not necessary
to achieve the same rate as long as the loss is Lipschitz continuous.

Note also that this result gives a similar upper bound as in \cite{bartlett2017spectrally},
albeit slightly looser.
One possible advantage of our bounds relative to theirs is
that,
according to Corollary~\ref{cor:transportation_based_risk_bound},
it gives risk bounds of both deterministic and stochastic predictors.
Hence, it may result in better risk bounds by optimizing the trade-off with respect to the amount of noise.

\section{CONCLUDING REMARKS}
\label{sec:conclusion}

We have presented the PAC-Bayesian transportation bound,
unifying the PAC-Bayesian analysis and the chaining analysis
with an infinitesimal limit.
It allows us to relate existing PAC-Bayesian bounds to the risk of deterministic predictors
by evaluating the cost of noise reduction.
As an example, we have given an upper bound on the noise reduction cost of neural networks,
which have given a negative answer to the necessity of the noise (and hence margins) in the recently-proposed PAC-Bayesian risk bound
for spectrally normalized neural networks.

One of the most significant implications of the transportation bound,
besides the de-randomization viewpoint,
is that it allows us to evaluate the risk of predictors
not only with what they are, but also with how we found them,
regarding the processes of parameter search as transportations.
This may be useful to analyze the generalization errors of iterative algorithms like SGDs.

As future work, we highlight two possible directions.
One is the characterization of the optimal posterior flow,
aiming to understand the geometry of the metric space $(\Pi(\Fcal), d_{S,\Ub,\delta})$.
Although a simple linear flow suffices to prove the results presented in this paper,
this direction of study may give more tight transportation bound.
It is also valuable to study whether the optimal posterior flow $\xi$
is approximated with a function of empirical values,
which hopefully links the transportation bound with new optimization algorithms.


\bibliographystyle{apalike}
\bibliography{reference}

\ifx\TRUNCATEAPP\undefined
\newpage

\appendix

In this appendix,
we provide the proofs for Theorem~\ref{thm:transportation_bound},
Theorem~\ref{thm:noise_reduction_for_neural_network} and Corollary~\ref{cor:total_deviation_of_neural_networks}
respectively.
First, in Section~\ref{sec:app_norm_based_notation},
we introduce the norm-based notation of the velocities $V_t(\xi;S)$ and $W_t(\xi)$ for convenience.
Second, in Section~\ref{sec:app_main_proof}, we prove the main result,
Theorem~\ref{thm:transportation_bound}.
Then, in Section~\ref{sec:app_neural_networks},
we give the proofs of the statements on the generalization error of neural networks, namely Theorem~\ref{thm:noise_reduction_for_neural_network} and Corollary~\ref{cor:total_deviation_of_neural_networks}.

\section{NORM-BASED NOTATION OF VELOCITIES}
\label{sec:app_norm_based_notation}

In this section, we introduce the formal notion of \emph{metric} and \emph{metric field}
as tools for measuring the velocity.

\head{Metrics}
Let $\Scal_+(\Fcal)$ denote the set of the symmetric nonnegative linear mappings on $\Fcal$,
i.e., if $A\in\Scal_+(\Fcal)$, $\inner{f}{Ag}=\inner{Af}{g}$ and $\inner{f}{Af}\ge 0$
for all $f,g\in\Fcal$.
We referred to such $A\in \Scal_+(\Fcal)$ as a \emph{metric}
as it defines a distance on $\Fcal$ as follows.
Let $\abs{\cdot}_\Lambda:\Fcal\to \RR_+$ denote the Mahalanobis' distance
with respect to $\Lambda\in \Scal_+(\Fcal)$, i.e.,
$\abs{h}_\Lambda=\sqrt{\inner{h}{\Lambda h}}$ for all $h\in\Fcal$.

\head{Metric Fields}
Next, we consider \emph{metric fields} on $\Fcal$,
a mapping of $\Fcal$ to $\Scal_+(\Fcal)$.
As metrics define the norm of vectors,
the metric fields define the norms of vector fields.

\begin{definition}
    Let $\Qb\in \Pi(\Fcal)$.
    Let $\mu:\Fcal\to \Fcal$ be a vector field on $\Fcal$
    and $\Lambda:\Fcal\to \Scal_+(\Fcal)$ be a metric field on $\Fcal$.
    Then, the $L^2$-norm of $\mu$ with respect to $\Lambda$ and $\Qb$ 
    is given by
    \begin{align*}
        \norm{\mu}_{\Lambda, \Qb}
        &\eqdef \norm{\abs{\mu}_{\Lambda}}_{L^2(\Qb)}
        =\sqrt{\Qb \abs{\mu}_{\Lambda}^2},
    \end{align*}
    where $\abs{\mu}_{\Lambda}$ denotes the function $f\mapsto \abs{\mu(f)}_{\Lambda(f)}$.
\end{definition}
Here, we implicitly assume that $\mu$ and $\Lambda$ are regular enough in the sense that
there exists an appropriate sigma algebra on $\Fcal$ that ensures the measurability of 
the integrand $\abs{\mu}_{\Lambda}$.

Finally, we introduce two metric fields, $\Lambda_S$ and $\Sigma$,
utilized in the main analysis.
One is
(i) a metric field induced from the distribution of the derivatives of loss functions,
defined as follows.

\begin{definition}[Data-dependent metric fields]
    Let $\Lambda_S:\Fcal\to \Scal_+(\Fcal)$ be metric fields
    given by
    \begin{align*}
        \Lambda_S(f)&\eqdef \frac{\Pb+\Pb_S}{2}\sbr{\nabla\Delta(f)\otimes \nabla\Delta(f)},
    \end{align*}
    where $\otimes$ denotes the tensor product such that $f\otimes g: \Fcal\to \Fcal$, $h \mapsto\inner{g}{h} f$,
    for all $f,g\in \Fcal$.
\end{definition}

The other is (ii) a metric field
denoted by $\Sigma:\Fcal\to \Scal_+(\Fcal)$, which is expected to be faithful to the data-independent structure of the loss function in the sense of 
the following assumption, which is equivalent to Assumption~\ref{asm:lipschitz_base}.

\begin{assumption}[Lipschitz condition, norm based]
    \label{asm:lipschitz}
    The loss function is Lipschitz continuous with respect to $\Sigma$,
    i.e., $\gamma(\Sigma)\eqdef \sup_{f\in\Fcal,z\in\Xcal}\abs{\nabla\Delta(f,z)}_{\Sigma(f)}^2< \infty$.
\end{assumption}

Note that, under this assumption, Assumption~\ref{asm:lipschitz_base} holds with $L_\Delta=\sqrt{\gamma(\Sigma)}$.
Moreover, the two velocities are written as
\begin{align*}
    V_t(\xi;S)&=\norm{\mu_t}_{\Lambda_S,\Qb_t^\xi},
    \\
    W_t(\xi)&=\norm{\mu_t}_{\Sigma^{-1},\Qb_t^\xi}.
\end{align*}
This way, it is easier to recall that these velocity measures satisfy homogeneity and triangle inequality.

\section{COMPLETE PROOF OF THEOREM~\ref{thm:transportation_bound}}
\label{sec:app_main_proof}

In this section, we present the detailed proofs for the main result.
First, in Section~\ref{sec:app_basic_pac_bayes}, we show an abstract PAC-Bayesian bound,
where the notion of \emph{centrality functions} plays an crucial role,
and then introduce several instance of centrality functions in Section~\ref{sec:app_centrality}.
Finally, in Section~\ref{sec:app_main_main},
we specialize the abstract bound into our transportation setting and
prove the main theorem.

\subsection{Abstract PAC-Bayesian bound}
\label{sec:app_basic_pac_bayes}

In this subsection, we introduce
fundamental inequalities of 
the PAC-Bayesian analysis.
Namely, the strong Fenchel duality of the KL divergence and log-integral-exp functions,
and its applications.

\begin{lemma}[\cite{donsker1975asymptotic}]
    \label{lem:duality_KL}
    For any measurable space $\Fcal$ and all $\Qb,\Ub\in\Pi(\Fcal)$,
    \begin{align*}
        \kld(\Qb,\Ub)=\sup_{X:\Fcal\to \RR} \Qb X - \ln \Ub\sbr{e^X},
    \end{align*}
    where the supremum is taken over all the measurable functions $X$.
    Therefore, we have
    \begin{align*}
        \Qb X\le \kld(\Qb,\Ub) + \ln \Ub\sbr{e^X}
    \end{align*}
    and it is impossible to improve it uniformly.
\end{lemma}
\begin{proof}
    Assume that $\Qb$ is not absolutely continuous with respect to $\Ub$.
    Then, the LHS diverges.
    Note that there exists $E\subset \Fcal$ such that $\Qb(E)>0$ and $\Ub(E)=0$.
    Hence, the RHS also diverges by taking $X=\alpha \1_E$ with $\alpha\to \infty$.

    On the other hand, if $\Qb$ is absolutely continuous with respect to $\Ub$, 
    there exists the Radon--Nikodym derivative $\frac{\d \Qb}{\d \Ub}$.
    Now, let $X=\ln \frac{\d \Qb}{\d \Ub}+Y$.
    Then we have
    \begin{align*}
        \Qb X - \ln \Ub\sbr{e^X}
        &=
        \Qb \sbr{\ln \frac{\d \Qb}{\d \Ub}+Y} - \ln \Ub\sbr{\frac{\d \Qb}{\d \Ub}e^Y}
        \\
        &=
        \kld(\Qb, \Ub) + \Qb Y - \ln \Qb\sbr{e^Y}
        \\
        &\le
        \kld(\Qb, \Ub).
        \\&(\text{Jensen's inequality})
    \end{align*}
    The equality holds if $Y$ is constant $\Qb$-almost surely.
\end{proof}

Utilizing Lemma~\ref{lem:duality_KL},
we present an abstract PAC-Bayesian inequality.
To this end, we introduce the notion of \emph{centrality}.
\begin{definition}[Centrality]
    We say a stochastic process $X$ is $\eta$-central,
    $\eta:\Fcal\times \Zcal\to \RR$,
    if
    \begin{align*}
        \Pb \sbr{e^{X(f)- \eta(f)}}\le 1
    \end{align*}
    for all $f\in \Fcal$.
    Moreover, we call $\eta$ as a centrality function of $X$.
\end{definition}
The centrality functions allow us to estimate how stochastic processes deviate from their expected values.
For example,
$X$ is $\sigma$-subgaussian
if and only if $\lambda X$ is $\frac{\lambda^2\sigma^2}{2}$-central for all $\lambda\in \RR$.
In particular, the centrality is essential in the PAC-Bayesian analysis
as showcased in the following lemma.

\begin{lemma}[Abstract PAC-Bayesian bound]
    \label{lem:abstract_pac_bayesian}
    Let $X:\Fcal\times \Zcal\to \RR$ be an arbitrary $\eta$-central process.
    Fix any prior distributions $\Ub\in \Pi(\Fcal)$.
    Then, 
    simultaneously for all $\Qb\in\Pi(\Fcal)$,
    \begin{align}
        \Qb \Pb_S X
        &\le \Qb \Pb_S \eta + \frac1{n} H_\delta(\Qb,\Ub)
        \label{eq:abstract_pac_bayesian}
    \end{align}
    with probability $1-\delta$ on the draw of $S\sim \Pb^n$,
    where
    $H_\delta(\Qb,\Ub)\eqdef \kld(\Qb, \Ub)+\ln \frac1\delta$
    denotes the complexity of $\Qb$ with respect to $\Ub$.
\end{lemma}

\begin{proof}
    To combine Lemma~\ref{lem:duality_KL} with the centrality property of $X$,
    we consider $X'(f,z)=X(f,z)-\eta(f,z)$.
    Thus, Lemma~\ref{lem:duality_KL} implies that
    \begin{align*}
        \Qb \Pb_S X
        &= \Qb \Pb_S\eta + \frac1{n}\Qb \Pb_S nX'
        \\
        &\le \Qb \Pb_S\eta
        + \frac1{n}\kld(\Qb,\Ub) + \frac1{n}\ln \Ub\sbr{e^{\Pb_S nX'}}.
    \end{align*}
    Now, the centrality condition combined with Markov's inequality yields
    \begin{align*}
        \Ub\sbr{e^{\Pb_S nX'}}
        &\le
        \frac{1}{\delta}\EE_{S\sim \Pb^n}\Ub\sbr{e^{\Pb_S nX'}}
        &(\text{Markov})
        \\
        &=
        \frac{1}{\delta}\prod_{i=1}^n \Pb\Ub\sbr{e^{X'}}
        \le \frac1\delta
        &(\text{centrality})
    \end{align*}
    with probability $1-\delta$, i.e.,
    \begin{align*}
        \Qb \Pb_S X
        &\le \Qb \Pb_S \eta + \frac1{n} H_\delta(\Qb,\Ub).
    \end{align*}
\end{proof}

One can confirm that Lemma~\ref{lem:abstract_pac_bayesian}
is an abstraction of the PAC-Bayesian analysis
as it derives one of the most standard PAC-Bayesian bound.
\footnote{The following theorem is complementary; It is just showing the connection of Lemma~\ref{lem:abstract_pac_bayesian} and the ordinary PAC-Bayesian bounds, and hence readers may skip it.}

\begin{theorem}
    \label{thm:standard_pac_bayesian}
    Let $\Delta:\Fcal\times \Zcal\to \RR$ be an arbitrary $\sigma$-subgaussian process,
    i.e., $\Pb e^{\lambda X(f)}\le \frac{\sigma^2\lambda^2}{2}$ for all $\lambda \in \RR$ and $f\in\Fcal$.
    Fix any prior distributions $\Ub\in \Pi(\Fcal)$.
    Then, 
    simultaneously for all $\Qb\in\Pi(\Fcal)$,
    \begin{align*}
        \Qb\Pb_S \Delta
        &\le \sigma \sqrt{
            2
            \frac{H_\delta(\Qb,\Ub) + \ln 2\sqrt{en}}{n-2},
        }
    \end{align*}
    with probability $1-\delta$ on the draw of $S\sim \Pb^n$.
\end{theorem}
\begin{proof}
    Let $X((\lambda, f), z)=\lambda\Delta(f,z)$ and
    take the prior $\Qbtil(\d \lambda \d f)=\Nb(\d \lambda)\Qb(\d f)$ and
    the posterior $\Ubtil(\d \lambda \d f)=\Ucal_1(\d\lambda)\Ub(\d f)$.
    Here, $\Nb=\Ncal(\lambdabar, \rho^2)\in\Pi(\RR)$ is a normal distribution with mean $\lambdabar$ and variance $\rho^2$
    and $\Ucal_1\in\Pi(\RR)$ is the standard Cauchy distribution given by
    \begin{align*}
        \frac{\Ucal_1(\d t)}{\d t}=\frac{1}{\pi(1+t^2)}.
    \end{align*}
    Since $\Delta/\sigma$ is $1$-subgaussian,
    $X/\sigma$ is by definition $\eta$-central
    for $\eta((\lambda, f), z)=\frac{\lambda^2}{2}$.
    Hence Lemma~\ref{lem:abstract_pac_bayesian}
    with the posterior $\Qbtil$ and prior $\Ubtil$
    implies that, with probability $1-\delta$,
    \begin{align*}
        \frac{1}{\sigma}\Qb\Pb_S \Delta
        &= \frac{1}{\lambdabar}\Qbtil \Pb_S \frac{X}{\sigma}
        \\
        &\le \frac{1}{\lambdabar}\Qbtil\Pb_S \eta + \frac1{n\lambdabar} H_\delta(\Qbtil,\Ubtil)
        \\
        &= \frac{(\lambdabar+\frac{\rho^2}{\lambdabar})}{2}
        +\frac1{n\lambdabar} \sbr{H_\delta(\Qb,\Ub) +\kld(\Nb,\Ucal_1)}
    \end{align*}
    for simultaneously all $\lambdabar, \rho>0$ and $\Qb\in\Pi(\Fcal)$.
    Now, take $\rho=\frac{1}{\sqrt{n}}$
    and observe
    \begin{align*}
        \kld(\Nb,\Ucal_1)
        &\le \ln \pi (1+\lambdabar^2 + \rho^2) -\frac12 \ln2\pi e\rho^2
        \\
        &\le \ln \frac{1+\lambdabar^2+\rho^2}{\rho}
        \\
        &\le \ln (1+\lambdabar^2) + \ln \rbr{\rho+\frac{1}\rho}
        \\
        &\le \lambdabar^2 + \ln 2\sqrt{n}
    \end{align*}
    where the first line follows from Jensen's inequality.
    Putting this back to the original position, we have
    \begin{align*}
        \frac1\sigma\Qb\Pb_S \Delta
        &\le \lambdabar\rbr{\frac{1}{2}+\frac{1}{n}}
        +\frac{H_\delta(\Qb,\Ub) + \ln 2\sqrt{n} + \frac12}{n\lambdabar},
    \end{align*}
    which is minimized as
    \begin{align*}
        \frac1\sigma\Qb\Pb_S \Delta
        &\le \sqrt{
            2\rbr{1+\frac{2}{n}}
            \frac{H_\delta(\Qb,\Ub) + \ln 2\sqrt{en}}{n},
        }
        \\
        &\le \sqrt{
            2
            \frac{H_\delta(\Qb,\Ub) + \ln 2\sqrt{en}}{n-2},
        }
    \end{align*}
    taking $\lambdabar=\sqrt{\frac{H_\delta(\Qb,\Ub)+\ln 2\sqrt{n}+\frac12}{n\rbr{\frac12+\frac1n}}}$.
\end{proof}

Thus, the centrality function $\eta$ controls the scale of the upper bound.

\subsection{Centrality Characterizations}
\label{sec:app_centrality}

Now, we show several examples of centrality functions
based on different characterizations of stochastic processes $X:\Fcal\times \Zcal\to \RR$.
The following two propositions are owing to Hoeffding's and Bennett's lemma, respectively.

\begin{proposition}[Hoeffding's centrality]
    \label{prop:hoeffding_centrality}
    Suppose that $a(f)\le X(f, z)\le b(f)$ uniformly for $a,b:\Fcal\to \RR$.
    Then, $X$ is $\frac{(b-a)^2}{8}$-central.
\end{proposition}

\begin{proposition}[Bennett's centrality]
    \label{prop:bennett_centrality}
    Suppose that $X(f, z)\le b(f)$ uniformly
    and let $\phi(x)=e^x-x-1$.
    Then, $X$ is $\frac{\phi(b)}{b^2}\Pb X^2$-central.
\end{proposition}

We omit the proofs (See \cite{boucheron2013concentration} for example).
These centralities have been typically used in the PAC-Bayesian literature so far
\footnote{Although Hoeffding's centrality is more popular,
Bennett's centrality also appears in \cite{seldin2012pac,tolstikhin2013pac} for example.
}.
However, since the uniform boundedness condition is inconvenient in our setting,
we are motivated to consider a centrality characterization
without any uniform boundedness conditions.

\begin{proposition}[Rademacher centrality]
    \label{prop:rademacher_centrality}
    Suppose that $X$ is centered,
    i.e., $\Pb X(f)=0$ for all $f\in \Fcal$.
    Then, $X$ is $\frac{X^2+\Pb X^2}2$-central.
\end{proposition}
\begin{proof}
    Let $\eta(f,z)=X^2(f,z)+\Pb X^2(f)$ and $D(f, z, z')\eqdef (X(f, z)-X(f, z'))^2$.
    Note that $\eta(f,z)=\frac12 \EE_{z'\sim \Pb} D(f,z,z')$.
    Moreover, let us introduce a Rademacher variable $\sigma$ that take $\pm 1$ with the same probability $1/2$ independently.
    Then, as inequalities of functions over $\Fcal$, we have
    \small
    \begin{align*}
        &\mathop{\EE}_{z\sim \Pb} \sbr{e^{X(z)-\eta(z)}}
        \\
        &\le
        \mathop{\EE}_{(z,z')\sim\Pb^2}\sbr{e^{X(z) - X(z') -\frac{1}{2}D(z,z')}}
        \\&(\text{Jensen's ineq.})
        \\
        &=
        \mathop{\EE}_{\substack{(z,z')\sim\Pb^{2}\\ \sigma\in\cbr{-1,+1}^n}}
        \sbr{e^{ \sigma (X(z)-X(z'))
        -\frac{1}{2}D(z,z')}}
        \\&(\text{Symmetry})
        \\
        &=
        \mathop{\EE}_{\substack{(z,z')\sim\Pb^{2}}}
        \sbr{e^{-\frac{1}{2}D(z,z')}\cosh \rbr{X(z)-X(z')}}
        \\
        &\le 1,
    \end{align*}
    \normalsize
    where the last inequality follows from the fact that $\cosh x \le e^{\frac12 x^2}$.
\end{proof}

Note that the result is corresponding to a deviation inequality given by~\cite{audibert2007combining},
but the major coefficient of the centrality function is halved
compared to theirs, $X^2+\Pb X^2$.

\subsection{PAC-Bayesian Transportation Bound}
\label{sec:app_main_main}

Now we apply the abstract PAC-Bayesian bound~(Lemma~\ref{lem:abstract_pac_bayesian}) to evaluate
the infinitesimal increment of deviation $\d D_t(\xi;S)$
in our transportation setting.

To this end, we first construct a posterior distribution.
Let us define a posterior $\Qbtil_t^\xi\in\Pi(\RR\times \Fcal^2)$ by
\begin{align}
    \Qbtil^{\xi}_t(\d \alpha \d f\d h)
    =
    \Ncal(\alpha;\alphabar,\sigma^2)
    \Qb^\xi_t(\d f)
    \Ncal(\d h;\mu_t(f),e^{\alpha}\Sigma(f))
    \label{eq:increment_posterior}
\end{align}
for $\alphabar\in \RR$ and $\sigma^2>0$.
Here, $\Ncal(\mu,\Sigma) \in \Pi(\Fcal)$ is a Gaussian distribution on $\Fcal$
with mean $\mu\in\Fcal$ and variance $\Sigma\in \Scal_+^1(\Fcal)$,
i.e.,
$\Ncal(\mu,\Sigma) \sbr{\inner{\cdot}{h}}=\inner{\mu}{h}$
and 
$\Ncal(\mu,\Sigma) \sbr{\inner{\cdot}{h}^2}=\inner{\mu}{h}^2+\inner{h}{\Sigma h}$
for all $h\in \Fcal$.

The following lemma shows that
the increment $\frac{\d}{\d t} D_t(\xi;S)$
can be transformed into a bilinear pairing of the above posterior and a data dependent function.

\begin{lemma}
    \label{lem:pac_bayesina_connection}
    Let $X:((\alpha, f,h),z)\mapsto\inner{h}{\nabla\Delta(f,z)}$.
    Then, with $\Qbtil_t^\xi$ given in \eqref{eq:increment_posterior},
    \begin{align*}
        \frac{\d}{\d t}D_t(\xi;S)=\Qbtil^{\xi}_t \Pb_S X.
    \end{align*}
\end{lemma}
\begin{proof}
    Note that, by Definition~\ref{def:ODE},
    \begin{align*}
        \frac{\d}{\d t} D_t(\xi;S)
        &= \frac{\d}{\d t} \EE_{f_0\sim \Qb_0} \sbr{\Pb_S \Delta(f_t)}
        \\
        &= \EE_{f_0\sim \Qb_0,\omega_0\sim P_0}
        \sbr{\Pb_S \inner{\xi_t(f_t, \omega_0)}{\nabla \Delta(f_t)}}
        \\
        &= \Qb_t^\xi\Pb_S \inner{\mu_t}{\nabla \Delta}.
    \end{align*}
    Moreover, we have
    \begin{align*}
        &\Qb_t^\xi\Pb_S \inner{\mu_t}{\nabla \Delta}
        \\
        &=\EE_{f\sim \Qb^\xi_t} \sbr{\Pb_S\inner{\mu_t(f)}{\nabla\Delta(f)}}
        \\
        &=\EE_{f\sim \Qb^\xi_t, h\sim \Ncal(0,e^\alpha\Sigma(f))}
        \sbr{\Pb_S \inner{\mu_t(f)+h}{\nabla\Delta(f)}}
        \\
        &=\Qbtil^{\xi}_t \Pb_S X,
    \end{align*}
    where the second inequality follows from the linearity of the inner product.
\end{proof}

Next, we introduce the corresponding prior distribution.
Let us define $\Ubtil\in\Pi(\RR\times \Fcal^2)$ by
\begin{align}
    \Ubtil(\d \alpha\d f\d h)=\Ucal_1(\d \alpha;\sigma)\Ub(\d f)\Ncal(\d h;0,e^{\alpha}\Sigma(f))
    \label{eq:increment_prior}
\end{align}
for some $\Ub\in\Pi(\Fcal)$.
Here, $\Ucal_1(\sigma)$ denotes the Cauchy distribution with the scale parameter $\sigma>0$,
\begin{align*}
    \Ucal_1(\d \alpha;\sigma)\eqdef \frac{\d \alpha}{\sigma\pi(1 + \alpha^2/\sigma^2)}.
\end{align*}
Note that
\begin{align}
    &\kld(\Qbtil_t^\xi, \Ubtil)
    \nonumber
    \\
    &=
    \kld(\Ncal(\alphabar, \sigma^2), \Ucal_1(\sigma))+
    \kld(\Qb_t^\xi, \Ub)+
    \nonumber
    \\
    &
    \quad\mathop{\EE}_{
        \substack{
            f\sim \Qb_t^\xi\\
            \alpha\sim \Ncal(\alphabar,\sigma^2)
        }
    }
    \kld\sbr{\Ncal(\mu_t(f),e^{\alpha}\Sigma(f)),\;\Ncal(0,e^{\alpha}\Sigma(f))}
    \nonumber
    \\
    &=
    \kld(\Ncal(\alphabar, \sigma^2), \Ucal_1(\sigma))+
    \kld(\Qb_t^\xi, \Ub)+
    \frac{\norm{\mu_t}^2_{\Sigma^{-1},\Qb_t^\xi}}{2e^{\alphabar-\sigma^2/2}}.
    \nonumber
    \\
    &\le
    \kld(\Qb_t^\xi, \Ub)+
    \ln \rbr{2+\frac{\alphabar^2}{\sigma^2}}+
    \frac{\norm{\mu_t}^2_{\Sigma^{-1},\Qb_t^\xi}}{2e^{\alphabar-\sigma^2/2}}.
    \label{eq:kl_main}
\end{align}

Now we are ready to prove the main theorem.
Below, we restate the theorem for convenience
with the norm-based notation introduced in Section~\ref{sec:app_norm_based_notation}.

\begin{theorem}[Transportation bound, restated]
    \label{thm:transportation_bound_re}
    Fix any prior distribution $\Ub\in \Pi(\Fcal)$.
    Then, 
    with probability $1-\delta$ on the draw of $S\sim \Pb^n$,
    we have
    \begin{align*}
        &\frac{\d}{\d t}D_t(\xi;S)
        \\
        &\; \le
        2\norm{\mu_t}_{\Lambda_S,\Qb_t^\xi}
        \sqrt{
            \frac{H_\delta(\Qb_t^\xi,\Ub)+c(n)}{n}
        }
      \\&
      \quad +2\norm{\mu_t}_{\Sigma^{-1},\Qb_t^\xi}
      \sqrt{\frac{\gamma(\Sigma)}{n}}.
    \end{align*}
    simultaneously for all posterior flow $\xi$,
    where $c(n)\eqdef \ln (e + \ln^2 2n^2)$.
\end{theorem}
\begin{proof}
    First of all, we assume $\gamma(\Sigma)=1$ without loss of generality~(if not,
    normalize $\Sigma$ dividing it by $\gamma(\Sigma)$).
    Take $X$ as in Lemma~\ref{lem:pac_bayesina_connection}.
    Note that $\Pb X((\alpha,f,h))=0$ for all $\alpha\in \RR, f,h\in \Fcal$ since $\Pb\Delta\equiv 0$
    and hence the Rademacher centrality~(Proposition~\ref{prop:rademacher_centrality}) implies that
    $X$ is $\eta$-central with $\eta=X^2+\Pb X^2$.
    Let $\Qbtil_t^\xi,\Ubtil\in\Pi(\RR\times\Fcal^2)$ be the posterior and prior distribution given in
    \eqref{eq:increment_posterior} and \eqref{eq:increment_prior}, respectively.
    Then, by Lemma~\ref{lem:abstract_pac_bayesian}, Lemma~\ref{lem:pac_bayesina_connection} and \eqref{eq:kl_main},
    \begin{align*}
        &\frac{\d}{\d t}D_t(\xi;S)
        \\
        &=\Qbtil^\xi_t \Pb_S X
        \\
        &\le \Qbtil^\xi_t \Pb_S \eta + \frac1{n} H_\delta(\Qbtil^\xi_t,\Ubtil)
        \\
        &\le \mathop{\EE}_{
            \substack{
                f\sim \Qb_t^\xi\\
                \alpha\sim \Ncal(\alphabar,\sigma^2)\\
                h\sim \Ncal(0,\Sigma(f))
            }
        }\abs{\mu_t(f)+e^{\alpha/2}h}_{\Lambda_S(f)}^2
        + \frac1{n} H_\delta(\Qb_t^\xi,\Ub)
        \\
        &\quad +\frac1{n} \cbr{
            \ln \rbr{2+\frac{\alphabar^2}{\sigma^2}}+
            \frac{\norm{\mu_t}^2_{\Sigma^{-1},\Qb_t^\xi}}{2e^{\alphabar-\sigma^2/2}}
        }.
    \end{align*}
    Now, observe that $\Qb_t^\xi{\tr(\Sigma\Lambda_S)}= \Qb_t^\xi\frac{\Pb+\Pb_S}{2}\abs{\nabla\Delta}_{\Sigma}^2\le \gamma(\Sigma)= 1$ and hence
    \begin{align*}
        &\mathop{\EE}_{
            \substack{
                f\sim \Qb_t^\xi\\
                \alpha\sim \Ncal(\alphabar,\sigma^2)\\
                h\sim \Ncal(0,\Sigma(f))
            }
        }\abs{\mu_t(f)+e^{\alpha/2}h}_{\Lambda_S(f)}^2
        \\
        &=\norm{\mu_t}_{\Lambda_S,\Qb_t^\xi}^2+
        \mathop{\EE}_{
            \substack{
                \alpha\sim \Ncal(\alphabar,\sigma^2)\\
            }
        } \sbr{e^{\alpha}}
        \Qb_t^\xi\sbr{\tr \Sigma \Lambda_S}
        \\
        &\le {\norm{\mu_t}_{\Lambda_S,\Qb_t^\xi}^2+e^{\alphabar+\sigma^2/2}}.
    \end{align*}
    Combining above with
    $\alphabar=\alphabar^*_t\eqdef \frac12\ln (\norm{\mu_t}_{\Sigma^{-1}, \Qb_t^\xi}^2/2n)$
    and
    $\sigma^2=\frac14$,
    we have, with probability $1-\delta$,
    \begin{align*}
        &\frac{\d}{\d t}D_t(\xi;S)
        \rbr{=\Qbtil^\xi_t \Pb_S X}
        \\
        &\le \norm{\mu_t}_{\Lambda_S,\Qb_t^\xi}^2
        +
        e^{1/8}\norm{\mu_t}_{\Sigma^{-1},\Qb_t^\xi}\sqrt{\frac{2}{n}}
        + 
        \\
        &\quad
        \frac1{n} \cbr{
            H_\delta(\Qb_t^\xi,\Ub)+
            \ln \rbr{e+\ln^2\frac{\norm{\mu_t}_{\Sigma^{-1}, \Qb_t^\xi}^2}{2n}}
        }
    \end{align*}
    for simultaneously all $\Qb_t^\xi\in\Pi(\Fcal)$ and $\mu_t:\Fcal\to \Fcal$,
    where we have loosen the bound with the inequality $2<e$.
    Now, multiply $\mu_t$ by $\lambda/\norm{\mu_t}_{\Sigma^{-1},\Qb_t^\xi}$
    with $\lambda>0$ and divide the both sides with $\lambda$ to obtain
    \begin{align*}
        &{\norm{\mu_t}_{\Sigma^{-1},\Qb_t^\xi}^{-1}}\frac{\d}{\d t}D_t(\xi;S)
        \\
        &\le \lambda R_t^2
        + 
        \frac{H_\delta(\Qb_t^\xi,\Ub)+l_n(\lambda)}{n\lambda}
        +
        e^{1/8}\sqrt{\frac{2}{n}}
        ,
    \end{align*}
    where $R_t\eqdef \norm{\mu_t}_{\Lambda_S,\Qb_t^\xi}/\norm{\mu_t}_{\Sigma^{-1},\Qb_t^\xi}$
    and $l_n(\lambda)\eqdef\ln (e + \ln^2\frac{\lambda^2}{2n})~(\ge 1)$.
    Then, taking $\lambda=\lambda^*_t$
    such that\footnote{
        Note that $\lambda_t^*$ always exists by the intermediate value theorem.
    }
    \begin{align*}
        \lambda^*_t &= \inf\myset{\lambda>0}{\sqrt{\frac{{ H_{\delta}(\Qb_t^\xi,\Ub)+l_n(\lambda)} }{ n\Rbar_t^2 }}= \lambda},
    \end{align*}
    where $\Rbar_t=R_t\vee \frac{1}{\sqrt{8n}}$,
    we have
    \begin{align}
        &{\norm{\mu_t}_{\Sigma^{-1},\Qb_t^\xi}^{-1}}\frac{\d}{\d t}D_t(\xi;S)
        \nonumber
        \\
        &\le
        2R_t
        \sqrt{
            \frac{H_\delta(\Qb_t^\xi,\Ub)+l_n(\lambda^*_t)}{n}
        }+
        \nonumber
        \\
        &\quad
        \frac{1}{\sqrt{8n}}\sqrt{
            \frac{H_\delta(\Qb_t^\xi,\Ub)+l_n(\lambda^*_t)}{n}
        }+
        e^{1/8}\sqrt{\frac{2}{n}}
        .
        \label{eq:pre_final}
    \end{align}

    Now we bound $l_n(\lambda^*_t)$ from above.
    Note that
    \begin{align*}
        0\le R_t^2
        &=\frac{\Qb_t^\xi\abs{\mu_t}^2_{\Lambda_S}}{\Qb_t^\xi\abs{\mu_t}^2_{\Sigma^{-1}}}
        \\
        &\le \sup_{f\in\Fcal} \frac{\abs{\mu_t(f)}^2_{\Lambda_S(f)}}{\abs{\mu_t(f)}^2_{\Sigma^{-1}(f)}}
        \\
        &\le \sup_{f\in\Fcal} \norm{\Lambda_S^{1/2}(f)\Sigma^{1/2}(f)}^2_2
        \\
        &\le \sup_{f\in \Fcal} \norm{\Lambda_S^{1/2}(f)\Sigma^{1/2}(f)}^2_{\rm F}
        \\
        &= \sup_{f\in \Fcal} \tr {\Lambda_S(f)\Sigma(f)}
        \\
        &\le \gamma(\Sigma)=1,
    \end{align*}
    where $\norm{A}_{\rm F}=\sqrt{\tr A^\top A}$ denote the Frobenius norm of $A$,
    which implies $\frac1{8n}\le \Rbar_t^2 \le 1$
    and hence
    \begin{align}
        \lambda_t^{*2}
        \ge 
        \frac{l_n(\lambda_t^*)}{n\Rbar_t^2}
        \ge 
        \frac{ 1 }{ n }
        \label{eq:lambda_lower}
    \end{align}
    by the definition of $\lambda_t^*$.
    Moreover, we have a trivial non-statistical bound
    \begin{align*}
        \frac{\d}{\d t}D_t(\xi;S)
        &=\Qb_t^\xi \Pb_S \inner{\mu_t}{\nabla\Delta}
        \\
        &\le\sqrt{\Qb_t^\xi \Pb_S \inner{\mu_t}{\nabla\Delta}^2}
        \\
        &\le\sqrt{2\Qb_t^\xi \frac{\Pb+\Pb_S}{2} \inner{\mu_t}{\nabla\Delta}^2}
        \\
        &=\sqrt{2}\norm{\mu_t}_{\Lambda_S,\Qb_t^\xi},
    \end{align*}
    which, compared to \eqref{eq:pre_final}, allows us to confine ourselves to\footnote{
        Otherwise, we get a worse risk bound than the trivial bound
        and \eqref{eq:pre_final} holds unconditionally.
    }
    \begin{align}
        \sqrt{
            2\frac{H_\delta(\Qb_t^\xi,\Ub)+l_n(\lambda^*_t)}{n}
        }\le 1.
        \label{eq:confinement}
    \end{align}
    Once again, by the definition of $\lambda^*_t$,
    this implies
    \begin{align}
        \lambda_t^{*2}={ \frac{H_\delta(\Qb_t^\xi, \Ub) + l_n(\lambda_t^*)}{n\Rbar_t^2} }\le \frac{1}{2\Rbar_t^2}\le 4n.
        \label{eq:lambda_upper}
    \end{align}
    To sum up the inequalities on $\lambda_t^*$, \eqref{eq:lambda_lower} and \eqref{eq:lambda_upper},
    we have
    \begin{align*}
        l(\lambda_t^*)
        &= \ln \rbr{e+\ln^2\frac{\lambda_t^{*2}}{2n}}
        \\
        &\le \ln \rbr{e+\ln^22n^2}
        = c(n).
    \end{align*}

    Finally, noting that $1/4+e^{1/8}\sqrt{2}\approx 1.85< 2$, the sum of the last two terms on the RHS of~\eqref{eq:pre_final} is bounded with
    $\frac{2}{\sqrt{n}}\norm{\mu_t}_{\Sigma^{-1},\Qb_t^\xi}$
    owing to \eqref{eq:confinement}.
    This concludes the proof.
\end{proof}

Note that, strictly speaking,
the image of the metric field $\Sigma$ should be in the trace class, $\tr \Sigma(f)\le \infty$ for all $f\in \Fcal$,
in order to ensure the existence of the distribution
$\Ncal(\mu_t(f),e^{\alpha}\Sigma(f))$ in \eqref{eq:kl_main}.
However, as in the definition of $\Sigma$ given in the main body,
this condition can be relaxed to the case of general metric fields $\Sigma(f)\in \Scal_+(\Fcal)$, $\forall f\in\Fcal$.
To see this, for any non-trace class metric fields $\Sigma:\Fcal\to \Scal_+(\Fcal)$,
take a trace class sequence $\cbr{\Sigma_i}_{i=1}^\infty$
such that
$\abs{h}_{\Sigma_i^{-1}(f)}$ is monotone decreasing for all $f,h\in\Fcal$
and $\abs{h}_{\Sigma_i^{-1}(f)} \to \abs{h}_{\Sigma^{-1}(f)}$ as $i\to \infty$ in a point-wise manner.
For example,
$\Sigma_i(f)=\sum_{k=1}^\infty (\lambda_k(f) \wedge 2^{i-k}) e_k(f)\otimes e_k(f)$
where $\cbr{\lambda_k(f)}_{k=1}^\infty$ and $\cbr{e_k(f)}_{k=1}^\infty$
are the eigenvalues and eigenvectors of $\Sigma(f)$.\footnote{
    Such countable eigen decomposition is always possible because $\Fcal$ is separable.
}
Then, we have $\tr \Sigma_i = 2^i < \infty$ and
Theorem~\ref{thm:transportation_bound_re} holds for each $\Sigma=\Sigma_i$.
In other words, for all $i\ge 1$, $\Pb^n \sbr{\chi_{m_i}}\le \delta$, where
\begin{align*}
    \chi_{F}(S)&\eqdef 1\cbr{F(S)> 0},\qquad S\in\Zcal^n,~F:\Zcal^n\to \RR,
    \\
    m_i(S)&\eqdef \sup_{\xi,t\in T} \cbr{\frac{\d}{\d t}D_t(\xi;S)-\iota_t(\xi;S,\Ub,\delta,\Sigma=\Sigma_i)}.
\end{align*}
Therefore, since $m_i$ and $\chi_{m_i}$ is increasing with respect to $i\ge 1$,
we have
by the monotone convergence theorem
\begin{align*}
    \delta
    &\ge \lim_{i\to \infty}\Pb^n \sbr{\chi_{m_i}}
    \\
    &=\Pb^n \sbr{\lim_{i\to \infty}\chi_{m_i}}
    \\
    &=\Pb^n \sbr{1_{m_\infty}},
\end{align*}
where $m_\infty\eqdef \sup_{\xi,t\in T} \cbr{\frac{\d}{\d t}D_t(\xi;S)-\iota_t(\xi;S,\Ub,\delta,\Sigma=\Sigma)}$.

\section{PROOFS OF THE RISK BOUNDS FOR NEURAL NETWORKS}
\label{sec:app_neural_networks}

In this section,
we first prove Theorem~\ref{thm:noise_reduction_for_neural_network} in Section~\ref{sec:app_nn_noise_reduction}
and then give a short proof of Corollary~\ref{cor:total_deviation_of_neural_networks}
in Section~\ref{sec:app_nn_risk_bound}
on the basis of the previous result.
Section~\ref{sec:app_nn_utility} provides a utility lemma used to prove
Theorem~\ref{thm:noise_reduction_for_neural_network}.

\subsection{Proof of Theorem~\ref{thm:noise_reduction_for_neural_network}}
\label{sec:app_nn_noise_reduction}

To prove Theorem~\ref{thm:noise_reduction_for_neural_network},
we start with showing upper bounds on the norms
$\abs{\cdot}_{\Lambda_S(\cdot)}$
and $\abs{\cdot}_{\Sigma^{-1}(\cdot)}$, and then
on $\norm{\cdot}_{\Lambda_S, \Qb_t^\xi}$
and $\norm{\cdot}_{\Sigma^{-1}, \Qb_t^\xi}$.
Then we construct the corresponding bounds in the increment bound $\iota_t(\xi;S,\Ub,\delta)$ with an appropriate configuration
of the posterior flow $\xi$,
which are utilized to bound the de-randomization cost $d_{S,\Ub,\delta}(\cdot, \cdot)$.

Take any $w\in\RR^d$ and $k\in[K]$.
We denote by $w[k]=W_k\in \RR^{m\times m}$ the $k$-th weight matrix of the network $f_w$.
Let $\lambda_k(w)=\norm{w[k]}_2$ be the spectral norm of $w[k]$
and let $\lambda(w)=\rbr{\sum_{k=1}^K \lambda_k^2(w)}^{1/2}$ be the aggregated spectral norm of $w$.
Let $\phi_{p,q;k}(w)=\norm{w[k]^\top}_{p,q}$ be the $(p,q)$-norm of $w[k]^\top$,
i.e., $\phi_{p,q;k}(w)=\rbr{\sum_{i=1}^m\rbr{\sum_{j=1}^m \abs{w[k]_{ij}}^p}^{\frac{q}{p}}}^{\frac1q}$.
Specifically, we reserve the shorthand 
$\phi_k(w)\eqdef \phi_{2,2;k}(w)$
for the Frobenius norm.

\subsubsection{Norm bounds}
Let $s_{\uparrow k}(x,w^{k-1})\in \RR^m$
be the normalized forward signal fed into the $k$-th layer,
\begin{align*}
    s_{\uparrow k}(x,w)
    &=\frac{a_{k-1}\circ  W_{k-1}\circ \cdots \circ a_1(x)}{R_{\Xcal}\prod_{l=1}^{k-1} \lambda_l(w)},
\end{align*}
Note that $\abs{s_{\uparrow k}(x,w)}\le 1$ for all $x\in\Xcal, w\in \RR^d$ and $k\in [K]$.
Then, we have upper bounds on the norms of interest.

\begin{lemma}[Norm bounds for NNs]
    \label{lem:norm_bound_NN}
    There exists a metric field $\Sigma:\Fcal\to \Scal_0^+(\Fcal)$ with $\gamma(\Sigma)\le 1$
    such that,
    for all $w,v\in \RR^{d}$,
    \begin{align*}
        \abs{v}_{\Sigma^{-1}(w)}&\le 2R(w) \sqrt{
            K\sum_{k=1}^K \frac{\phi_k^2(v)}{\lambda_k^2(w)}
        }.
    \end{align*}
    Moreover,
    for all $w,v\in \RR^{d}$,
    \begin{align*}
        \abs{v}_{\Lambda_S(w)}&\le 2R(w) \sqrt{
            K\sum_{k=1}^K \frac{\frac{3\Pb+\Pb_S}{4}\abs{
                    v[k]\;s_{\uparrow k}(w)
            }^2}{\lambda_k^2(w)}
        }
        \\
        &\le 2R(w) \sqrt{
            K\sum_{k=1}^K \frac{\lambda_k^2(v)}{\lambda_k^2(w)}
        }.
    \end{align*}
\end{lemma}
\begin{proof}
    The first claim follows if we take
    $\Sigma=\diag \cbr{\Sigma_k}_{k=1}^K$
    with $\Sigma_k(w)\eqdef \frac{\lambda_k^2(w)}{4KR^2(w)}I_{m^2}$.
    
    To see $\gamma(\Sigma)\le 1$, let $s_{\uparrow\downarrow k}(x,y,w)\in \RR^{m\times m}$ be the normalized backward signal fed into the $k$-th layer,
    \small
    \begin{align*}
        &s_{\uparrow\downarrow k}(x,y,w)
        \\
        &=\frac{
            \rbr{W_K\circ J_K\circ \cdots \circ W_{k+1}\circ J_{k+1}}^\top \nabla \ell(y,f_w(x))
        }{L_\ell \prod_{l=k+1}^K\lambda_k(w)},
    \end{align*}
    \normalsize
    where $J_k=\frac{\partial a_k(s)}{\partial s}_{s=s_{\uparrow k}(x,w^{k-1})}\in \RR^{m\times m}$
    is the Jacobian matrix of the $k$-th activation.
    Since $\norm{J_k}_2\le 1$,
    we have $\abs{s_{\uparrow \downarrow k}(x,y,w)}\le 1$.
    Now note that the gradient of the $k$-th layer is given by
    the product of backward and forward signals,
    $G_k(w,z)\eqdef \sbr{\nabla_w \ell(f_w, (x,y))}_k= \frac{R(w)}{\lambda_k(w)} s_{\uparrow\downarrow k}(x,y,w)\otimes s_{\uparrow k}(x,w)$.
    Therefore,
    we have
    \begin{align*}
        \gamma(\Sigma)
        &=
        \sup_{w\in\RR^d,z\in\Zcal} \abs{\nabla \Delta(f_w,z)}_{\Sigma(f)}^2
        \\
        &\le
        4\sup_{w\in\RR^d,z\in\Zcal} \abs{\nabla \ell(f_w,z)}_{\Sigma(f)}^2
        \\
        &=
        \sup_{w\in\RR^d,z\in\Zcal} \sum_{k=1}^K
        \frac{\lambda_k^2(w)\abs{G_k(w,z)}^2}
        {KR^2(w)}
        \\
        &=
        \sup_{w\in\RR^d,z\in\Zcal} \sum_{k=1}^K
        \frac{\abs{s_{\uparrow\downarrow k}(x,y,w)}^2\abs{s_{\uparrow k}(x,w)}^2}
        {K}
        \\
        &\le 1,
    \end{align*}
    where we utilized $R(w)=L_\ell R_\Xcal \prod_{k=1}^K\lambda_k(w)$.

    Now we prove the second claim, the bound on $\abs{v}_{\Lambda_S(w)}$.
    Observe that $\nabla\Delta(f_w,z)=\cbr{G_k(w,z)-\Pb G_k(w)}_{k=1}^K$ and hence
    \begin{align*}
        &\abs{v}_{\Lambda_S(w)}^2
        \\
        &=\frac{\Pb+\Pb_S}{2} \rbr{\sum_{k=1}^K \inner{v[k]}{G_k(w)-\Pb G_k(w)}}^2
        \\
        &\le 2K\frac{\Pb+\Pb_S}{2} \sum_{k=1}^K \rbr{\inner{v[k]}{G_k(w)}^2+\inner{V_k}{\Pb G_k(w)}^2}
        \\
        &\le 4K\frac{3\Pb+\Pb_S}{4} \sum_{k=1}^K \inner{v[k]}{G_k(w)}^2,
    \end{align*}
    where both inequalities follow from Jensen's inequality exchanging the order of expectation/summation and square.
    Finally, by the Cauchy--Schwarz' inequality,
    we have
    \begin{align*}
        \inner{V_k}{G_k(w,z)}
        &=R(w) \frac{s_{\uparrow\downarrow k}(x,y,w)^\top v[k] \;s_{\uparrow k}(x,w)}
        {\lambda_k(w)}
        \\
        &\le R(w) \frac{\abs{v[k]\; s_{\uparrow k}(x,w)}}
        {\lambda_k(w)}
        \\
        &\le R(w) \frac{\lambda_k(v)}
        {\lambda_k(w)},
    \end{align*}
    which concludes the proof.
\end{proof}

\subsubsection{$L^2$-norm bounds}

Next,
we specify posterior flows $\xi$ and
evaluate the corresponding $L^2$-norms.
Let $\Qb_0\in\Pi(\RR^d)$ be the initial posteriors
with mean $\wbar_0\in\RR^d$ 
and (diagonal) standard deviatioin $s_0\in \RR^d_+$,
\begin{align}
    \Qb_0=\Ncal(\wbar_0, {\diag}^2 s_0).
    \label{eq:initial_posterior}
\end{align}
Then, consider the contraction flow given by
\begin{align}
    \xi_t(w)=\mu_t(w)=\wbar_\infty-w
    \label{eq:contraction_flow}
\end{align}
for some $\wbar_\infty\in\RR^d$.
Note that $\Qb_t^\xi=\Ncal(\wbar_t, \diag^2 s_t)$
with $\wbar_t\eqdef \wbar_\infty+e^{-t}(\wbar_0-\wbar_\infty)$
and $s_t\eqdef e^{-t}s_0$.

We measure the $k$-th spectral radius of the flow $\xi$ at time $t\in T$ 
with
\begin{align*}
    \lambdabar_{t,k}(\xi)\eqdef \lambda_k(\wbar_t)+
    e^{-t} \gamma_m \psi_k(s_0),
\end{align*}
where
$\gamma_m=\sqrt{2\ln 2em}$
and 
$\psi_k(s_0)\eqdef \norm{s_0[k]}_{2,\infty}\vee \norm{s_0[k]^\top}_{2,\infty}$.
Let $\Rbar_t(\xi)$ denote the total spectral radius of $\xi$ at time $t\in T$,
$\Rbar_t(\xi)\eqdef L_\ell R_\Xcal \prod_{k=1}^K \lambdabar_{t,k}(\xi)$.

\begin{lemma}[$L^2$-norm bounds with contraction]
    \label{lem:l2_norm_bound_NN_contraction}
    Let $t\in T$.
    Let $\lambdabar_{t,k}(\xi)$ and $\Rbar_t(\xi)$ be given as above.
    Take the initial distribution and the posterior flow as in~\eqref{eq:initial_posterior}
    and \eqref{eq:contraction_flow}.
    Then, there exists a metric field $\Sigma:\Fcal\to \Scal_0^+(\Fcal)$ with $\gamma(\Sigma)\le 1$
    such that
    \begin{align*}
        &\norm{\mu_t}_{\Sigma^{-1},\Qb_t^\xi}
        \\
        &\le 2e^{-t}
        \Rbar_t(\xi) 
        \sqrt{
            K\sum_{k=1}^K \frac{\phi_k^2(\wbar_\infty-\wbar_0)+\phi_k^2(s_0)}{\lambdabar_{t,k}^2(\xi)}
        }.
    \end{align*}
    Moreover,
    \begin{align*}
        &\norm{\mu_t}_{\Lambda_S,\Qb_t^\xi}
        \\
        &\le 2e^{-t}
        \Rbar_t(\xi) 
        \sqrt{
            K\sum_{k=1}^K \frac{\lambda_k^2(\wbar_\infty-\wbar_0)+\phi_{\infty,2;k}^2(s_0)}{\lambdabar_{t,k}^2(\xi)}
        }.
    \end{align*}
\end{lemma}
\begin{proof}
    Note that $w_t\eqdef\wbar_\infty + e^{-t}(c+z)\sim \Qb_t^\xi$
    with $c=\wbar_0-\wbar_\infty$ and $z\sim \Ncal(0, \diag^2 s_0)$.
    Thus, by Lemma~\ref{lem:norm_bound_NN},
    \begin{align*}
        &\norm{\mu_t}_{\Sigma^{-1},\Qb_t^\xi}^2
        \\
        &=\EE_{w\sim \Qb_t^\xi} \abs{\wbar_\infty-w}_{\Sigma^{-1}(w)}^2
        \\
        &\le 4KL_\ell^2 R_\Xcal^2
        \sum_{k=1}^K
        e^{-2t}
        \EE\;
        \phi_k^2(c+z)
        \prod_{l\neq k}\lambda_{t,l}^2(w_t)
        .
    \end{align*}
    Note that the random parameter $z$ is layer-wise independent
    and Lemma~\ref{lem:gaussian_spectral_norm} yields, for all $l\in[K]$ and $t\in T$,
    \begin{align*}
        &\sqrt{\EE\lambda_l^2(w_t)}
        \\
        &=\sqrt{\EE\lambda_l^2(\wbar_\infty+e^{-t}(c+z))}
        \\
        &\le
        \lambda_l((1-e^{-t})\wbar_\infty+e^{-t}\wbar_0)
        + e^{-t}\psi_l(s_0)\sqrt{2\ln 2em}
        \\
        &= \lambdabar_{t,l}(\xi).
    \end{align*}
    Therefore, 
    the first claim is proved
    with the equality
    \begin{align*}
        \EE\;\phi_k^2(c+z)
        &= \phi_k^2(c) + \phi_k^2(s_0), \quad \forall k\in[K].
    \end{align*}
    Similarly, the second claim is also proved with
    \begin{align*}
        &\EE\abs{(c+z)[k]\; s_{\uparrow k}(x,w_t)}^2
        \\
        &=
        \EE\sbr{
            \abs{c[k]\; s_{\uparrow k}(x,w_t)}^2
    \right.\\&\qquad\quad\left.
            + 2\inner{c[k]\; s_{\uparrow k}(x,w_t)}{z[k]\; s_{\uparrow k}(x,w_t)}
    \right.\\&\qquad\quad\left.
            +\abs{z[k]\; s_{\uparrow k}(x,w_t)}^2
        }
        \\
        &=
        \EE\sbr{
            \abs{c[k]\; s_{\uparrow k}(x,w_t)}^2
            +\abs{z[k]\; s_{\uparrow k}(x,w_t)}^2
        }
        \\
        &\le
        \lambda^2_k(c)+\phi_{\infty,2;k}^2(s_0),\quad \forall x\in\Xcal,\;\forall k\in[K],
    \end{align*}
    where we utilized that $z[k]$ and $s_{\uparrow k}(x,w_t)$ are mutually independent
    as $s_{\uparrow k}(x,w_t)$ only depends on the layers $w_t[l]$ before the $k$-th layer,
    $l<k$, which are independent of $w_t[k]$.
\end{proof}

\subsubsection{De-randomization bound}

We now show the de-randomization bound, Theorem~\ref{thm:noise_reduction_for_neural_network},
on the basis of
Lemma~\ref{lem:l2_norm_bound_NN_contraction}.
To this end, we first specialize the general result of the lemma
to our simplified setting.
Then, we also evaluate a specific value of the entropy term $H_\delta(\Qb_t^\xi,\Ub)$
in order to compute the increment bound $\iota_t(\xi;S,\Ub,\delta)$.
Finally, we conclude the proof by integrating it from $t=0$ to $t=\infty$.

Take $\wbar_0=\wbar_\infty=w$ for $w\in\RR^d$.
Since the activation function is homogeneous,
we can assume $\lambda_1(w)=\ldots\lambda_K(w)=\Lbar(w)$
without loss of generality.
To have the initial posterior as specified, $\Qb_0=\Nb(w, \rho)=\Ncal(w, \rho^2\frac{\Lbar^2(w)}{mK^2\gamma_m^2}I_d)$,
take $s_0=\sigma\one_d$
for $\sigma=\frac{\Lbar(w) \rho}{\sqrt{m}K\gamma_m}$,
where $\one_d\in\RR^d$ is the vector of ones.
Then, we have $\phi_k(s_0)=\frac{\Lbar(w) \sqrt{m}}{K\gamma_m}\rho$, $\phi_{\infty,2;k}(s_0)=\psi_k(s_0)=\frac{\Lbar(w) \rho}{K\gamma_m}$.

Note that $\lambdabar_{t,k}(\xi)\le \lambda_{0,k}(\xi)=\Lbar(w)(1+\frac{\rho}{K})\le \Lbar(w)e^{\rho/K}$
and hence Lemma~\ref{lem:l2_norm_bound_NN_contraction} is simplified as
\begin{align}
    \norm{\mu_t}_{\Sigma^{-1},\Qb_t^\xi}
    &\le
    \rho e^{-t}
    \frac{
        2e^\rho R(w) 
    }
    {\gamma_m}
    \sqrt{m},
    \label{eq:tmp_2345vtrwe}
    \\
    \norm{\mu_t}_{\Lambda_S,\Qb_t^\xi}
    &\le \rho e^{-t}
    \frac{
        2e^\rho R(w) 
    }
    { \gamma_m}
    .
    \label{eq:tmp_23w34refefwe}
\end{align}

Next, we evaluate $H_\delta(\Qb_t^\xi,\Ub)$.
Let $\Ub=\Ucal_d$ be the generalized Cauchy distribution.
Then, noting that $\Qb_t^\xi=\Ncal(w, e^{-2t}\rho^2\frac{\Lbar^2(w)}{mK^2\gamma_m^2}I_d)$,
Lemma~\ref{lem:kl_gauss_cauchy} implies
\begin{align}
    H_\delta(\Qb_t^\xi,\Ub)
    &\le
    \frac{d+1}{2}\ln \rbr{1+\frac{\rho_0^2\abs{w}^2}{\rho^2e^{-2t}}}
    \nonumber\\
    &\quad+\ln \frac{\sqrt{d}}{\delta}\rbr{\frac{\rho e^{-t}}{\rho_0} +\frac{\rho_0}{\rho e^{-t}}}
    ,
    \label{eq:tmp_adsfadsffdgfbbsv}
\end{align}
where $\rho_0\eqdef \sqrt{\frac{K}{m}}\frac{\gamma_m}{\Lbar(w)}$.

Let $c_1(n,d,\delta)=c(n)+\ln \frac{\sqrt{d}}{\delta}$
and $c_2(n,d,\delta,u)=c_1(n,d,\delta) + 1 + \ln (u+1/u)$.
Then, 
with \eqref{eq:tmp_2345vtrwe}, \eqref{eq:tmp_23w34refefwe} and \eqref{eq:tmp_adsfadsffdgfbbsv}
combined with Corollary~\ref{cor:transportation_based_risk_bound}
and the inequality $\sqrt{a+b}\le \sqrt{a}+\sqrt{b}$,
we have
\small
\begin{align*}
    &d_{S,\Ub,\delta}\rbr{\Nb(w,\rho), \deltab_w}
    \\
    &\le
    \frac{
        4e^\rho R(w) 
    }
    { \gamma_m \sqrt{n}}
    \Bigg[
        \rho\sqrt{m}+
  \\&\quad
        \int_0^\infty \d t\;\rho e^{-t} \Bigg\{
            \sqrt{
                \frac{d+1}{2}\ln \rbr{1+\frac{\rho_0^2\abs{w}^2}{\rho^2e^{-2t}}}
            }+
  \\&\quad\qquad \qquad \qquad \;
            \sqrt{
                \ln \rbr{\frac{\rho e^{-t}}{\rho_0} +\frac{\rho_0}{\rho e^{-t}}}
                + c_1(n,d,\delta)
            }
        \Bigg\}
    \Bigg]
    \\
    &=
    \frac{
        4e^\rho R(w) 
    }
    { \gamma_m \sqrt{n}}
    \Bigg[
        \rho\sqrt{m} +
  \\&\qquad
  \rho_0 \abs{w}\sqrt{d}\; \Ical\rbr{\frac{\rho}{\rho_0\abs{w}}}
        +\rho_0 \Jcal\rbr{\frac{\rho}{\rho_0}, c_1(n,d,\delta)}
    \Bigg]
    \\
    &\le
    4 e^\rho R(w)
    \sqrt{\frac{mK^2}{n}}
    \\
    &\quad 
    \times
    \sbr{
        \frac{\abs{w}}{\Lbar(w)} \;\Ical\rbr{\sqrt{\frac{m}{K}}\frac{\Lbar(w)\rho}{\abs{w}\gamma_m}} +
\right.\\&\qquad \quad\left.
        \frac{\rho}{K\gamma_m} \rbr{
            1+\sqrt{\frac{c_2(n,d,\delta,\rho/\rho_0)}{m}}
        }
    }
    ,
\end{align*}
\normalsize
where $\Ical(a)\eqdef \int_0^a \d s\sqrt{\frac{1+1/d}{2}\ln(1+1/s^2)}$
and 
$\Jcal(a, b)\eqdef \int_0^a \d s\sqrt{\ln (s+1/s)+b}$.
The last inequality is shown by Jensen's inequality,
\begin{align*}
    \Jcal(a, b)
    &\le a\sqrt{\frac1a\int_0^a \d s\; \ln \rbr{s+\frac1{s}}+b}
    \\
    &= a\sqrt{\ln \rbr{a+\frac1a}-1+2\frac{\arctan a}{a}+b}
    \\
    &\le a\sqrt{\ln \rbr{a+\frac1a}+1+b}.
\end{align*}

Finally,
noting that $\Lbar(w)\le \abs{w}/\sqrt{K}$ and $\frac{1+1/d}{2}\le 1$,
Theorem~\ref{thm:noise_reduction_for_neural_network} is proved.

\subsection{Proof of Corollary~\ref{cor:total_deviation_of_neural_networks}}
\label{sec:app_nn_risk_bound}

Corollary~\ref{cor:total_deviation_of_neural_networks}
is shown combining Theorem~\ref{thm:noise_reduction_for_neural_network}
with a standard PAC-Bayesian risk bound in view of
the transportation based risk bound given as Corollary~\ref{cor:transportation_based_risk_bound}.

\begin{proof}
    According to \cite{mcallester1999some},
    for simultaneously all $\Qb_0\in \Pi(\Fcal)$,
    \begin{align}
        \Qb_0 \Pb_S \Delta
        =\tilde\Ocal\rbr{\sqrt{\frac{H_\delta(\Qb_0, \Ub)}{n}}}
        \label{eq:conventional_pac_bound_mca}
    \end{align}
    with probability at least $1-\delta$,
    where $\Ub\in \Pi(\Fcal)$ is an arbitrary data-independent prior.

    Now, take $\Qb_0=\Nb(w, 1)=\Ncal(w, \frac{\Lbar^2(w)}{mK^2\gamma_m^2}I_d)$ and $\Ub=\Ucal_d$.
    Then, according to Lemma~\ref{lem:kl_gauss_cauchy},
    we have
    \begin{align*}
        H_\delta(\Qb_0, \Ub)
        &=\Ocal\rbr{\frac{mK^2\gamma_m^2\abs{w}^2}{2\Lbar^2(w)}}
    \end{align*}
    for sufficiently large $m$.
    Therefore, inserting this back to~\eqref{eq:conventional_pac_bound_mca},
    we have
    \begin{align*}
        \Qb_0 \Pb_S \Delta
        =\tilde\Ocal\rbr{\frac{\abs{w}}{\Lbar(w)}\sqrt{\frac{mK^2}{n}}},
    \end{align*}
    which, combined with Corollary~\ref{cor:transportation_based_risk_bound}
    and Theorem~\ref{thm:noise_reduction_for_neural_network},
    concludes the proof..
\end{proof}

\subsection{Utility Lemmas}
\label{sec:app_nn_utility}

\begin{lemma}[Expected Gaussian spectral norms]
    \label{lem:gaussian_spectral_norm}
    Let $G=(g_{ij})$ be a $m\times m$ random matrices whose entries $g_{ij}$ are
    drawn from independent Gaussian distributions with mean $m_{ij}$ and variance $v_{ij}\ge 0$ for $i,j\in [m]$.
    Then,
    \begin{align*}
        \sqrt{\EE \norm{G}_2^2}
        &\le
        \norm{M}_2+\gamma_m\sqrt{v},
    \end{align*}
    where $M=(m_{ij})$ and 
    $v=\max_i\sum_j v_{ij}\vee \max_j \sum_i v_{ij}$.
\end{lemma}
\begin{proof}
    Note that $(a+b)^2=\inf_{(p,q)\in A} pa^2+qb^2$ for all $a,b\ge 0$,
    where $A\eqdef \mysetinline{(p,q)\in(1,\infty)^2}{\frac1p+\frac1q\le 1}$.
    Therefore, the triangle inequality implies
    \begin{align*}
        \sqrt{\EE \norm{M+G}_2^2}
        &\le
        \sqrt{\EE \cbr{\norm{M}_2+\norm{G}_2}^2}
        \\
        &\le
        \sqrt{\inf_{p,q\in A}p\norm{M}_2^2+q\EE \norm{G}_2^2},
        \\
        &=
        \norm{M}_2+\sqrt{\EE \norm{G}_2^2},
    \end{align*}
    where the second line follows from the exchange of the order of infimum and expectation.
    This, combined with $\EE \norm{G}_2^2\le 2v(1+\ln 2m)$~(see Eq.~(4.1.7), \cite{tropp2015introduction}),
    verifies the claim.
\end{proof}

\begin{lemma}[Gaussian--Cauchy KL divergence]
    \label{lem:kl_gauss_cauchy}
    Let $\Ncal_d(\mu,\rho^2I_d)\in\Pi(\RR^d)$ be a $d$-dimensional spherical Gaussian distribution
    with mean $\mu\in\RR^d$ and variance $\rho^2>0$.
    Let $\Ucal_d\in\Pi(\RR^d)$ be a generalized Cauchy distribution
    given by
    \begin{align*}
        \Ucal_d(\d x)=\frac{2\nu_d(\d x)}{\pi S_{d-1}\rbr{\abs{x}^{d-1}+\abs{x}^{d+1}}},
    \end{align*}
    where $\nu_d$ is the Lebesgue measure on $\RR^d$
    and $S_{d-1}=\frac{d\sqrt{\pi}^d}{\Gamma(\frac{d}2+1)}$ is the surface area of the unit ball in $\RR^d$.
    Then,
    \begin{align*}
        &\kld(\Ncal(\mu,\rho^2I_d), \Ucal_d)
        \\
        &\le\frac{d+1}{2}\ln \rbr{1+\frac{\abs{\mu}^2}{d\rho^2}}
        +\ln \frac{1+d\rho^2}{\rho}
    \end{align*}
    Moreover,
    \begin{align*}
        \kld(\Ncal(\mu,\rho^2I_d), \Ucal_d)
        &\le
        \frac{\abs{\mu}^2 }{2\rho^2}
        + \ln\frac{1+ 2d\rho^2 }{\rho}.
    \end{align*}
\end{lemma}
\begin{proof}
    Let $\Nb=\Ncal(\mu,\rho^2I_d)$.
    Then, by Jensen's inequality, we have
    \begin{align*}
        &\kld(\Nb, \Ucal_d)
        \\
        &=\Nb\sbr{\ln \frac{\d\Nb}{\d\Ucal}}
        \\
        &=\Nb \sbr{\ln \frac{\pi S_{d-1}\rbr{\abs{\cdot}^{d-1}+\abs{\cdot}^{d+1}}}{2}}
        +\frac{d}{2}\ln 2\pi e \rho^2
        \\
        &=\Nb \sbr{\ln \abs{\cdot}^{d-1}+\ln \rbr{1+\abs{\cdot}^{2}}}+
        \ln \frac{\pi S_{d-1}}{2}-\frac{d}{2}\ln 2\pi e \rho^2
        \\
        &\le\frac{d-1}{2}\ln \rbr{\abs{\mu}^2+d\rho^2}+\ln \rbr{1+\abs{\mu}^2+d\rho^2}
        \\
        &\;+\ln \frac{\pi S_{d-1}}{2}-\frac{d}{2}\ln 2\pi e \rho^2.
    \end{align*}
    Moreover, according to Stirling's approximation,
    we have
    \begin{align*}
        \ln S_{d-1}
        &=\ln d+\frac{d}{2}\ln \pi-\ln \Gamma\rbr{\frac{d}{2}+1}
        \\
        &\le\ln d+\frac{d}{2}\ln \pi-\frac{d}{2}\ln\frac{d}{2e} - \frac12 \ln \pi d
        \\
        &=\frac12 \ln \frac{d}{\pi}+\frac{d}{2}\ln \frac{2\pi e}{d}.
    \end{align*}
    Combining above, we have
    \begin{align*}
        &\kld(\Nb, \Ucal_d)
        \\
        &\le\frac{d-1}{2}\ln \rbr{\abs{\mu}^2+d\rho^2}+\ln \rbr{1+\abs{\mu}^2+d\rho^2}
        \\
        &\quad+\frac12 \ln \frac{\pi d}{4}+\frac{d}{2}\ln \frac{1}{d \rho^2}.
        \\
        &\le\frac{d-1}{2}\ln \rbr{1+\frac{\abs{\mu}^2}{d\rho^2}}+
        \ln \frac{1+\abs{\mu}^2+d\rho^2}{\rho},
    \end{align*}
    where the last inequality follows from $\pi \le 4$.
    Now, the first claim is proved with
    \begin{align*}
        \ln \frac{1+\abs{\mu}^2+d\rho^2}{\rho}
        &=\ln \frac{1+d\rho^2}{\rho}\rbr{1+\frac{\abs{\mu}^2}{1+d\rho^2}}
        \\
        &\le
        \ln \rbr{1+\frac{\abs{\mu}^2}{d\rho^2}}+\ln \frac{1+d\rho^2}{\rho},
    \end{align*}
    whereas the second claim follows from $\ln (1+x)\le x~(x\ge 0)$,
    i.e.,
    $\frac{d-1}{2}\ln \rbr{1+\frac{\abs{\mu}^2}{d\rho^2}}\le \frac{(d-1)\abs{\mu}^2}{2d\rho^2}$ and
    \begin{align*}
        \ln \frac{1+\abs{\mu}^2+d\rho^2}{\rho}
        &\le \ln \frac{1+\abs{\mu}^2+2d\rho^2}{\rho}
        \\
        &\le
        \ln \rbr{1+\frac{\abs{\mu}^2}{2d\rho^2}}+\ln \frac{1+2d\rho^2}{\rho}
        \\
        &\le
        \frac{\abs{\mu}^2}{2d\rho^2}+\ln \frac{1+2d\rho^2}{\rho}.
    \end{align*}
\end{proof}

\fi
\end{document}